\pdfoutput=1
\documentclass{article}

\PassOptionsToPackage{numbers,compress}{natbib}

\newcommand{\paperdir}{.}

\IfFileExists{neurips.sty}{
  \usepackage[preprint]{neurips}
  \renewcommand{\paperdir}{.}
}{
  \PackageError{neurips}{Could not find neurips.sty}{Place neurips.sty in this directory.}
}

\usepackage[utf8]{inputenc} %
\usepackage[T1]{fontenc}    %
\usepackage{hyperref}       %
\usepackage{url}            %
\usepackage{booktabs}       %
\usepackage{amsfonts}       %
\usepackage{nicefrac}       %
\usepackage{microtype}      %
\usepackage[dvipsnames]{xcolor}         %
\usepackage{graphicx}
\usepackage{subcaption}
\usepackage{amssymb}
\usepackage{amsmath}
\usepackage{mathtools}
\usepackage{amsthm}
\usepackage[capitalize,noabbrev]{cleveref}
\usepackage{xspace}
\usepackage{aliascnt}
\usepackage{tikz}
\usetikzlibrary{arrows.meta, backgrounds, positioning, calc, fit, shapes, shapes.symbols, decorations.pathreplacing, decorations.pathmorphing}
\usepackage{enumitem}
\usepackage{makecell}

\theoremstyle{plain}
\newtheorem{theorem}{Theorem}[section]
\newaliascnt{proposition}{theorem}
\newtheorem{proposition}[proposition]{Proposition}
\aliascntresetthe{proposition}

\crefname{proposition}{proposition}{propositions}
\Crefname{proposition}{Proposition}{Propositions}

\newaliascnt{openproblem}{theorem}
\newtheorem{openproblem}[openproblem]{Open Problem}
\aliascntresetthe{openproblem}

\crefname{openproblem}{open problem}{open problems}
\Crefname{openproblem}{Open Problem}{Open Problems}

\theoremstyle{definition}
\newtheorem{definition}[theorem]{Definition}

\theoremstyle{remark}

\definecolor{icmlinstr}{RGB}{200,0,0}     %
\definecolor{llmteal}{RGB}{0,128,128}     %
\definecolor{giovannic}{RGB}{0,70,170}    %
\definecolor{eugenioc}{RGB}{0,120,60}     %
\definecolor{lorenzoc}{RGB}{210,120,0}    %
\definecolor{sashac}{RGB}{120,60,180}     %
\definecolor{michelec}{RGB}{170,0,110}    %

\newcommand{\stat}{\textrm{stat}}

\newcommand{\area}{\textrm{area}}
\newcommand{\bounce}{\textrm{bounce}}
\newcommand{\dinv}{\textrm{dinv}}
\newcommand{\skipstat}{\textrm{skip}}
\newcommand{\magnitude}{\textrm{mag}}  %
\newcommand{\leap}{\textrm{leap}}  %
\newcommand{\sk}{\textrm{skew}}  %
\newcommand{\mingarc}{\textrm{mingarc}}

\newcommand{\methodnamefunctional}{MapSeek-Functional\xspace}
\newcommand{\methodnamesymbolic}{MapSeek-Symbolic\xspace}

\renewcommand{\L}{\mathcal{L}}
\newcommand{\R}{\mathbb{R}}
\newcommand{\N}{\mathbb{N}}

\newcommand{\NC}{\mathsf{NC}}
\newcommand{\PP}{\mathsf{PP}}

\usepackage{xargs}

\newdimen\qtTikzRadiusDim
\newdimen\qtTikzScaleLimitDim

\makeatletter
\newcommandx{\disk}[5][1 = {0,0}, 2 = 0.7, 3 = {black}, 4 = {black}, usedefault]{%
	\def \p {#1}
	\def \r {#2}
	\def \c {#3}
	\def \f {#4}
	\def \n {#5}
	\@diski}

\newcommand\@diski{%
	\begin{scope}[on background layer]
		\draw (\p) [thin, black] circle (\r);
	\end{scope}
	\@diskii}

\newcommand\@diskii{\@ifnextchar\stop{\@diskend}{\@diskiii}}

\newcommand\@diskiii[1]{%
	\draw (\p) pic (a) {disk={\r}{\c}{\f}{\n}{#1}};
	\@diskii
}

\newcommand\@diskend[1]{}
\makeatother

\newcommandx{\block}[5][1 = {0,0}, 2 = 0.7, 3 = {black}, usedefault]{%
	\def \p {#1}
	\def \r {#2}
	\def \f {#3}
	\def \n {#4}

	\begin{scope}[on background layer]
		\draw (\p) [thin, black] circle (\r);
	\end{scope}

	\draw (\p) pic (a) {block={\r}{\f}{\n}{#5}};
}

\newcommandx{\dblock}[5][1 = {0,0}, 2 = 0.7, 3 = {black}, usedefault]{%
	\def \p {#1}
	\def \r {#2}
	\def \f {#3}
	\def \n {#4}

	\begin{scope}[on background layer]
		\draw (\p) [thin, black] circle (\r);
	\end{scope}

	\draw (\p) pic (a) {dblock={\r}{\f}{\n}{#5}};
}

\newcommandx{\labs}[5][1 = {0,0}, 2 = 0.7, 3 = {black}, usedefault]{%
	\def \p {#1}
	\def \r {#2}
	\def \c {#3}
	\def \n {#4}

	\begin{scope}[on background layer]
		\draw (\p) [thin, black] circle (\r);
	\end{scope}

	\draw (\p) pic (a) {labs={\r}{\c}{\n}{#5}};
}

\newcommandx{\dlab}[6][1 = {0,0}, 2 = {0.15}, 3 = {0.05}, 4 = 0.7, 5 = {black}, usedefault]{%
	\def \p {#1}
	\def \x {#2}
	\def \y {#3}
	\def \c {#5}

	\qtTikzRadiusDim=#4 cm\relax
	\qtTikzScaleLimitDim=0.7 cm\relax
	\ifdim\qtTikzRadiusDim<\qtTikzScaleLimitDim \def \s {#4/0.7} \else \def \s {1} \fi

	\filldraw[fill=\c] (\p) circle (2 pt) node[scale=0.8*\s, above] at (\x+\p+\y) {$#6$};
	\filldraw[fill=\c] (\p) circle (2 pt) node[scale=0.8*\s, below] at (-\x-0.05+\p-\y-0.05) {$-#6$};
}

\tikzset{
	pics/block/.style n args = {4}{ %
		code = {

			\qtTikzRadiusDim=#1 cm\relax

			\def \angle {360/#3}

			\filldraw[draw=black, fill=#2, fill opacity=0.3] \foreach \x/\l [count=\xi] in {#4}{
				\ifnum\xi=1 \else--\fi
				({90-\angle*(\x - 0.5)}:\the\qtTikzRadiusDim) node at ({90-\angle*(\x - 0.5)}:{\the\qtTikzRadiusDim+0.3 cm}) {}
			} -- cycle;
		}
	}
}

\tikzset{
	pics/labs/.style n args = {4}{ %
		code = {

			\qtTikzRadiusDim=#1 cm\relax

			\def \angle {360/#3}

			\qtTikzScaleLimitDim=0.7 cm\relax
			\ifdim\qtTikzRadiusDim<\qtTikzScaleLimitDim \def \s {#1/0.7} \else \def \s {1} \fi

			\filldraw[fill=#2] \foreach \x/\l in {#4} {
				({90-\angle*(\x - 0.5)}:\the\qtTikzRadiusDim) circle (2*\s pt) node[scale=0.8*\s] at ({90-\angle*(\x - 0.5)}:{\the\qtTikzRadiusDim+\s*0.3 cm}) {$\l$}
			};
		}
	}
}

\tikzset{
	pics/disk/.style n args = {5}{ %
		code = {

			\qtTikzRadiusDim=#1 cm\relax

			\def \angle {360/#4}

			\qtTikzScaleLimitDim=0.7 cm\relax
			\ifdim\qtTikzRadiusDim<\qtTikzScaleLimitDim \def \s {#1/0.7} \else \def \s {1} \fi

			\draw \foreach \x/\l [count=\xi] in {#5} {
				\ifnum\xi=1 \else--\fi
				({90-\angle*(\x - 0.5)}:\the\qtTikzRadiusDim)
			} -- cycle;

			\filldraw[draw=black, fill=#3, fill opacity=0.3] \foreach \x/\l [count=\xi] in {#5}{
				\ifnum\xi=1 \else--\fi
				({90-\angle*(\x - 0.5)}:\the\qtTikzRadiusDim) node at ({90-\angle*(\x - 0.5)}:{\the\qtTikzRadiusDim+\s*0.3 cm}) {}
			} -- cycle;

			\filldraw[fill=#2] \foreach \x/\l in {#5} {
				({90-\angle*(\x - 0.5)}:\the\qtTikzRadiusDim) circle (2*\s pt) node[scale=0.8*\s] at ({90-\angle*(\x - 0.5)}:{\the\qtTikzRadiusDim+\s*0.3 cm}) {$\l$}
			};
		}
	}
}

\tikzset{
	pics/dblock/.style n args = {4}{ %
		code = {

			\qtTikzRadiusDim=#1 cm\relax

			\def \angle {360/#3}

			\draw \foreach \x/\l [count=\xi] in {#4} {
				\ifnum\xi=1 \else -- \fi
				({90-\angle*(\x - 0.5)}:\the\qtTikzRadiusDim)
			} -- (0:0) -- cycle;

			\fill[opacity=0.3, color=#2] \foreach \x/\l [count=\xi] in {#4}{
				\ifnum\xi=1 \else--\fi
				({90-\angle*(\x - 0.5)}:\the\qtTikzRadiusDim) node at ({90-\angle*(\x - 0.5)}:{\the\qtTikzRadiusDim+0.3 cm}) {}
			} -- (0:0) -- cycle;
		}
	}
}

\graphicspath{{./}{\paperdir/}}

\title{\texorpdfstring{Mapping Uncharted Symmetries:\\Machine Discovery in Combinatorics}{Mapping Uncharted Symmetries: Machine Discovery in Combinatorics}}

\author{%
Eugenio Cainelli\thanks{Equal contribution.}\\
  University of Bologna\\
  eugenio.cainelli2@unibo.it
\And
Lorenzo Luccioli\\
  University of Bologna\\
  lorenzo.luccioli2@unibo.it
\And
Alessandro Iraci\\
  Pegaso University\\
  alessandro.iraci@unipegaso.it
\And
Michele D'Adderio\\
  University of Pisa\\
  michele.dadderio@unipi.it
\And
Giovanni Paolini\footnotemark[1]\\
  University of Bologna\\
  g.paolini@unibo.it
}

\begin{document}

\maketitle
\begin{abstract}
    Inspired by long-standing open problems in algebraic combinatorics, we show that modern machine learning can meaningfully contribute to verifiable mathematical discoveries.
    In particular, we focus on the construction of simple mathematical functions under exact distributional constraints, a setting we formalize as Simple Learning Under Rigid Proportions (SLURP).
    We tackle this problem by introducing two methods: \methodnamefunctional, which models the desired function alternating pseudo-labeling and supervised training steps; and \methodnamesymbolic, designed to directly produce symbolic formulas.
    We successfully apply both methods to a research problem in algebraic combinatorics, discovering a new combinatorial interpretation of the $q,t$-Narayana polynomials arising from representation theory.
    To our knowledge, this is the first such interpretation based on noncrossing partitions.
    Using one discovered statistic, we find a combinatorial proof of the symmetry of these polynomials in a previously unsolved case.
    To streamline verification and reproducibility, we release all code, including a formalization of all the mathematical discoveries of this paper in Lean 4.
\end{abstract}

\section{Introduction}

Many important problems in mathematics can be framed as function discovery: given two sets $X$ and $Y$, find a function $f\colon X\to Y$ that satisfies rigid constraints, yet admits a simple description.
Modern machine learning is exceptionally good at fitting finite datasets, and its inductive biases often favor low-complexity solutions.
In this paper we push these strengths to the limit: we work in settings where success requires \emph{exact} constraint satisfaction rather than high accuracy, and where the desired output is not a black-box predictor but a \emph{symbolic or algorithmic} description, revealing interesting structure and suitable for mathematical verification.

Our driving goal comes from open problems in the field of algebraic combinatorics, asking for functions on combinatorial objects that witness deep algebraic symmetries.
Inspired by such problems, we introduce a learning framework called \textit{Simple Learning Under Rigid Proportions} (SLURP), where the task is to discover simple functions that match specified output distributions exactly (\Cref{fig:slurp-diagram}).
In this setting, the constraints deliberately underspecify the function $f$, so that it is easy to produce example functions, but the challenge is to find simple and mathematically meaningful ones.

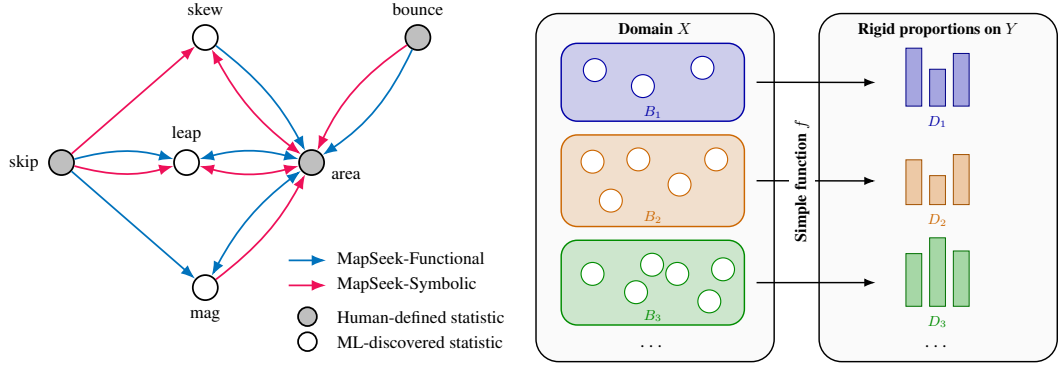
\begin{figure}[t]
  \begin{subfigure}[t]{0.495\linewidth}
    \centering
    \resizebox{\linewidth}{!}{%
      \begin{tikzpicture}[
        >=Latex,
        stat/.style={circle, draw, thick, minimum size=4mm, inner sep=0pt},
        humanstat/.style={stat, fill=gray!50},
        start/.style={stat, double, double distance=0.8pt},
        mone/.style={thick, RoyalBlue},
        mtwo/.style={thick, OrangeRed},
        other/.style={thick},
        lab/.style={font=\footnotesize}
      ]

      \node[humanstat, humanstat] (skip) at (0, 0) {};
      \node[lab] at ($(skip) + (-0.6, 0)$) {skip};

      \node[stat] (mag) at (2.3, -2) {};
      \node[lab] at ($(mag) + (0, -0.45)$) {mag};

      \node[stat] (skew) at (2.3, 2) {};
      \node[lab] at ($(skew) + (0, 0.45)$) {skew};

      \node[stat] (leap) at (2, 0) {};
      \node[lab] at ($(leap) + (0, 0.45)$) {leap};

      \node[humanstat] (area) at (4, 0) {};
      \node[lab] at ($(area) + (0.55, -0.2)$) {area};

      \node[humanstat] (bounce) at (5.7, 2) {};
      \node[lab] at ($(bounce) + (0, 0.45)$) {bounce};

      \draw[mone, ->] (skip) -- (mag);
      \draw[mone, ->] (skip) to[bend left=15] (leap);
      \draw[mone, <->] (leap) to[bend left=15] (area);
      \draw[mone, <->] (mag) to[bend left=12] (area);
      \draw[mone, ->] (skew) to[bend left=12] (area);
      \draw[mone, ->] (bounce) to[bend left=15] (area);

      \draw[mtwo, ->] (skip) to[bend right=15] (leap);
      \draw[mtwo, ->] (skip) to (skew);
      \draw[mtwo, <->] (area) to[bend left=15] (leap);
      \draw[mtwo, <->] (area) to[bend left=12] (skew);
      \draw[mtwo, ->] (bounce) to[bend right=15] (area);
      \draw[mtwo, ->] (mag) to[bend right=12] (area);

      \node %
      (leg)
        at (5.36, -2.2)
      {
      \begin{tikzpicture}[>=Latex, baseline=(base)]
        \node (base) {};

        \draw[mone, ->] (0,0.20) -- (0.6,0.20);
        \node[lab, anchor=west] at (0.65,0.20) {\methodnamefunctional};

        \draw[mtwo, ->] (0,-0.20) -- (0.6,-0.20);
        \node[lab, anchor=west] at (0.65,-0.20) {\methodnamesymbolic};

        \node[humanstat, minimum size=3mm] at (0.3, -0.70) {};
        \node[lab, anchor=west] at (0.65, -0.7) {Human-defined statistic};

        \node[stat, minimum size=3mm] at (0.3, -1.1) {};
        \node[lab, anchor=west] at (0.65, -1.1) {ML-discovered statistic};
      \end{tikzpicture}
      };

      \end{tikzpicture}
    }
  \end{subfigure}
  \hfill
  \begin{subfigure}[t]{0.495\linewidth}
    \centering
    \resizebox{\linewidth}{!}{%
      \begin{tikzpicture}[
        font=\small,
        >=Latex,
        panel/.style={draw, rounded corners=14, line width=0.9pt, fill=black!2},
        bag/.style={draw=#1, rounded corners=10, line width=0.9pt, fill=#1, fill opacity=0.18},
        nodept/.style={circle, draw, line width=0.7pt, minimum size=5.0mm, inner sep=0pt, fill=white},
        histbar/.style={draw, line width=0.6pt}
      ]

      \colorlet{bagA}{blue!65!black}
      \colorlet{bagB}{orange!80!black}
      \colorlet{bagC}{green!55!black}
      \node[panel, minimum width=5.2cm, minimum height=7.6cm, anchor=south west] (L) at (0,0) {};
      \node[font=\bfseries] at ($(L.north)+(0,-0.35)$) {Domain $X$};

      \node[panel, minimum width=5.2cm, minimum height=7.6cm, anchor=south west] (R) at ($(L.south east)+(0.95,0)$) {};
      \node[font=\bfseries] at ($(R.north)+(0,-0.38)$) {Rigid proportions on $Y$};

      \node[bag=bagA, anchor=north west, minimum width=4.0cm, minimum height=1.7cm] (bag1)
        at ($(L.north west)+(0.55,-0.65)$) {};
      \node[bag=bagB, anchor=north west, minimum width=4.0cm, minimum height=2.0cm] (bag2)
        at ($(L.north west)+(0.55,-2.65)$) {};
      \node[bag=bagC, anchor=north west, minimum width=4.0cm, minimum height=1.85cm] (bag3)
        at ($(L.north west)+(0.55,-4.95)$) {};

      \node[font=\small, text=bagA] at ($(bag1.south)+(0,0.22)$) {$B_1$};
      \node[font=\small, text=bagB] at ($(bag2.south)+(0,0.22)$) {$B_2$};
      \node[font=\small, text=bagC] at ($(bag3.south)+(0,0.22)$) {$B_3$};

      \node[nodept, draw=bagA] at ($(bag1.north west)+(0.75,-0.60)$) {};
      \node[nodept, draw=bagA] at ($(bag1.north west)+(1.80,-0.95)$) {};
      \node[nodept, draw=bagA] at ($(bag1.north west)+(3.10,-0.55)$) {};

      \node[nodept, draw=bagB] at ($(bag2.north west)+(0.70,-0.60)$) {};
      \node[nodept, draw=bagB] at ($(bag2.north west)+(1.70,-0.55)$) {};
      \node[nodept, draw=bagB] at ($(bag2.north west)+(2.60,-1.10)$) {};
      \node[nodept, draw=bagB] at ($(bag2.north west)+(3.40,-0.55)$) {};
      \node[nodept, draw=bagB] at ($(bag2.north west)+(1.10,-1.45)$) {};

      \node[nodept, draw=bagC] at ($(bag3.north west)+(0.70,-0.75)$) {};
      \node[nodept, draw=bagC] at ($(bag3.north west)+(1.65,-1.15)$) {};
      \node[nodept, draw=bagC] at ($(bag3.north west)+(2.55,-0.75)$) {};
      \node[nodept, draw=bagC] at ($(bag3.north west)+(3.25,-1.35)$) {};
      \node[nodept, draw=bagC] at ($(bag3.north west)+(2.0,-0.55)$) {};
      \node[nodept, draw=bagC] at ($(bag3.north west)+(3.55,-0.65)$) {};

      \coordinate (h1) at ($(R.north west)+(2.6,-1.52)$);
      \coordinate (h2) at ($(R.north west)+(2.6,-3.67)$);
      \coordinate (h3) at ($(R.north west)+(2.6,-5.89)$);

      \begin{scope}[scale=1.15]
        \draw[histbar, draw=bagA!80!black, fill=bagA!35] ($(h1)+(-0.60,-0.45)$) rectangle ++(0.30,1.10);
        \draw[histbar, draw=bagA!80!black, fill=bagA!35] ($(h1)+(-0.15,-0.45)$) rectangle ++(0.30,0.70);
        \draw[histbar, draw=bagA!80!black, fill=bagA!35] ($(h1)+(0.30,-0.45)$) rectangle ++(0.30,1.00);

        \draw[histbar, draw=bagB!80!black, fill=bagB!35] ($(h2)+(-0.60,-0.45)$) rectangle ++(0.30,0.85);
        \draw[histbar, draw=bagB!80!black, fill=bagB!35] ($(h2)+(-0.15,-0.45)$) rectangle ++(0.30,0.55);
        \draw[histbar, draw=bagB!80!black, fill=bagB!35] ($(h2)+(0.30,-0.45)$) rectangle ++(0.30,0.95);

        \draw[histbar, draw=bagC!80!black, fill=bagC!35] ($(h3)+(-0.60,-0.45)$) rectangle ++(0.30,1.00);
        \draw[histbar, draw=bagC!80!black, fill=bagC!35] ($(h3)+(-0.15,-0.45)$) rectangle ++(0.30,1.30);
        \draw[histbar, draw=bagC!80!black, fill=bagC!35] ($(h3)+(0.30,-0.45)$) rectangle ++(0.30,1.05);
      \end{scope}

      \node[font=\small, text=bagA] at ($(h1)+(0,-0.85)$) {$D_1$};
      \node[font=\small, text=bagB] at ($(h2)+(0,-0.85)$) {$D_2$};
      \node[font=\small, text=bagC] at ($(h3)+(0,-0.85)$) {$D_3$};

      \draw[->, line width=0.9pt] ($(bag1.east)+(0.25,0)$) -- ($(h1)+(-1.35,0)$);
      \draw[->, line width=0.9pt] ($(bag2.east)+(0.25,0)$) -- ($(h2)+(-1.35,0)$);
      \draw[->, line width=0.9pt] ($(bag3.east)+(0.25,0)$) -- ($(h3)+(-1.35,0)$);

      \node at ($(bag3.south)+(0,-0.45)$) {\large $\cdots$};
      \node at ($(h3)+(0,-1.39)$) {\large $\cdots$};

      \node [font=\bfseries, rotate=90, fill=white] at ($(bag2.east)+(1.25,0)$) {Simple function $f$};

      \end{tikzpicture}%
    }
  \end{subfigure}
  \caption[PCA clusters and entropy trace for MapSeek-Functional runs]{
    Left: atlas of notable statistics (nodes) for \(k=3\). An edge $f \to g$ indicates that, starting from statistic $f$, one of our methods discovers $g$ as a partner statistic, so that $(f, g)$ provides a combinatorial interpretation of Narayana polynomials $N_{n, 3}(q, t)$.
    Right: a SLURP instance where the domain $X$ is partitioned into bags, and supervision is provided via per-bag output distributions.
  }
  \label{fig:stat-atlas}
  \label{fig:slurp-diagram}
\end{figure}

We introduce two machine learning methods for SLURP: \methodnamefunctional and \methodnamesymbolic.
To demonstrate their effectiveness and the soundness of the SLURP setup, we describe a successful application %
to the combinatorial study of the $q,t$-Narayana polynomials $N_{n, k}(q, t)$ arising from representation theory.
In this context, an important problem is to find interpretations of $q,t$-polynomials in terms of pairs of simple functions (\textit{statistics}) on a set of combinatorial objects; often, one candidate statistic is known while a partner statistic is missing.
We find that both methods are effective in discovering statistics and yield new combinatorial interpretations of $N_{n, 3}(q, t)$.
Successful pairings are represented as directed edges in \Cref{fig:stat-atlas}.

One interpretation,
independently suggested by both our ML pipelines, leads us to a new combinatorial proof of the symmetry $N_{n, 3}(q, t) = N_{n, 3}(t, q)$ via an explicit bijection that exchanges the \textit{skip} and \textit{leap} statistics.
This constitutes one of the first combinatorial proofs of $q,t$-symmetry across all of algebraic combinatorics; constructing such exchanging bijections is widely regarded as an extremely difficult open problem.
To streamline verification, we formalize the full proof in the Lean~4 proof assistant, which is increasingly becoming a standard tool for formalizing mathematics.

The main contributions of this paper are as follows:
\begin{itemize}[leftmargin=*, itemsep=2pt, topsep=0pt, parsep=0pt, partopsep=0pt]

    \item Drawing inspiration from algebraic combinatorics problems, we formalize the SLURP framework.

    \item We introduce two methods for SLURP: a functional self-training method (\methodnamefunctional) and a symbolic search method (\methodnamesymbolic). The code will be open-sourced upon acceptance.

    \item We propose the task of discovering statistics as a testbed for SLURP.
    In particular, we release the \emph{skip-pairing benchmark}, where the aim is to discover interpretations of $q,t$-Narayana polynomials.

    \item We use \methodnamefunctional and \methodnamesymbolic to discover interpretable and mathematically interesting statistics, and build upon them to derive new combinatorial results.

    \item We construct an explicit exchanging bijection for the symmetry $N_{n,3}(q,t)=N_{n,3}(t,q)$, thus solving \Cref{prob:narayana-bijection} in the case $k=3$, and we provide the first interpretation of $N_{n,k}(q,t)$ in terms of noncrossing partitions. We formalize these results in Lean~4 and release the code.
\end{itemize}

Our broader goal is to build flexible tools that, in the hands of expert mathematicians, can leverage human intuition and extend their reach.

\section{Related Work}

\paragraph{Machine learning for mathematical discovery.}
The application of machine learning to mathematics has attracted significant interest in recent years. A substantial body of work focuses on automated theorem proving, often leveraging large language models \cite{aristotle, SeedProver}. Separately, a growing line of research employs ML not as a prover but as an exploratory tool for discovering mathematical objects, constructing counterexamples, and formulating conjectures.
Examples include searching for extremal graphs \cite{Wagner}, studying algebraic invariants \cite{PetarNature}, and synthesizing Lyapunov functions for dynamical systems \cite{Lyapunov}; see also \citep{Williamson}.
Our contribution stands within this paradigm, with the aim of constructing functions (combinatorial statistics and bijections) that satisfy exact constraints.

\paragraph{Automated discovery in $q,t$-combinatorics.} The search for combinatorial interpretations and exchanging bijections is a central challenge in algebraic combinatorics. Recently, ML has been applied to this specific domain: in \citep{zeta-map-AI}, a Transformer was trained to approximate the \textit{zeta map} \cite{HaglundBook}, recovering a new algorithmic interpretation of this known map by analyzing attention weights; in \citep{openevolve}, LLMs and evolutionary algorithms were used to generate bijection candidates. However, these contributions have primarily served as proofs of concept, re-deriving established results. To our knowledge, our work is the first to use ML to propose \emph{previously unknown} combinatorial statistics, used to construct an exchanging bijection that solves an open problem in $q,t$-combinatorics.

\paragraph{Learning from label proportions (LLP).} Our SLURP framework shares with LLP the use of bag-level distributional constraints rather than individual labels~\cite{DBLP:journals/jmlr/QuadriantoSCL09,ScottZhang2020,Havaldar2024,EasyLLP2023}.
However, the settings differ fundamentally.
In LLP, a fixed unknown distribution generates instances, and distributional constraints serve as a proxy for hidden labels; in contrast, for us distributional constraints are the full problem specification.
Moreover, LLP aims to learn a classifier with good generalization, typically allowing approximate constraint satisfaction, %
whereas SLURP requires exact satisfaction along with a concise symbolic or algorithmic description (and, in our applications,
solutions that can be verified on infinite domains).

\paragraph{Program synthesis (PS).}
The objective of recovering explicit symbolic formulas under symbolic constraints situates SLURP within the broader scope of PS \citep{PS, PS2, PS3}. Formally, SLURP can be cast as a PS instance by making the hypothesis class explicit, e.g.\ by defining a domain specific language or bounding syntactic complexity.
Nevertheless, SLURP differs from typical PS work because the domains are infinite and the constraints are distributional (bag-level) rather than pointwise.
Crucially, standard automated verification tools like SMT solvers cannot check constraint satisfaction over infinite domains.
We therefore present SLURP as a proof-driven discovery framework rather than a PS variant: we rely on machine learning to propose candidates from finite data, while the simplicity requirement ensures these candidates are amenable to the mathematical proof.

\section{Simple Learning Under Rigid Proportions}
\label{sec:slurp-methods}

Inspired by algebraic combinatorics, we introduce a new learning framework called \emph{Simple Learning Under Rigid Proportions} (SLURP).
This framework formalizes settings where the goal is to discover a \emph{simple} function $f \colon X \to Y$ that satisfies \emph{rigid proportion constraints}.
We first describe the problem formulation, and then our two approaches.

\subsection{Problem Formulation}

The SLURP constraints are expressed as follows: the domain $X$ is partitioned into bags $B_1, B_2, \dots$, and each bag $B_i$ comes with a target distribution $D_i$ on $Y$ (\Cref{fig:slurp-diagram}).
The goal is to find a function $f \colon X \to Y$ such that:

\begin{enumerate}[label=(\arabic*), itemsep=2pt, topsep=0pt, parsep=0pt, partopsep=0pt]
    \item The outputs on each bag $B_i$ are distributed according to $D_i$ exactly (\textbf{rigid proportion constraint}).

    \item The function $f$ is provided in the form of a symbolic description (e.g., a closed formula or algorithm), evaluable within a given time bound (e.g., polynomial), and as short as possible (\textbf{simplicity requirement}).

\end{enumerate}

The simplicity requirement is not merely aesthetic but a fundamental necessity. Without it, one could trivially satisfy the rigid proportion constraints by arbitrarily assigning values to domain elements in a way that matches the target distributions, without revealing any mathematical insight.

\paragraph{Key characteristics.}
SLURP differs from traditional learning settings, such as LLP, in several fundamental ways.
First, the constraints are \textit{rigid}, meaning that solutions must satisfy them exactly, not approximately.
While we can measure how close a function is to satisfying the constraints (to use it as a learning signal, or to track progress during learning), there is no notion of an ``almost correct'' solution. Any deviation from exact satisfaction means the problem remains unsolved.

Second, the aggregated constraints are not a weak proxy of a ground-truth solution, but rather the true end-goal of our framework. Thus, the problem is underdetermined and multiple simple functions may be viable solutions.
The challenge is to discover functions that are both feasible and mathematically interesting.

\paragraph{Learning on infinite domains.}
When the domain $X$ is infinite, as in the practical scenarios we tackle in this paper, we limit the learning process to only use a finite training subset $X_{\mathrm{train}}\subset X$ consisting of a union of a finite number of bags~$B_i$.
The symbolic description of a solution $f \colon X_{\mathrm{train}} \to Y$ to the finite SLURP problem can be interpreted as a candidate extension $\tilde f \colon X \to Y$; however, the only way to obtain a true guarantee that $\tilde f$ satisfies the rigid proportion constraints on the whole infinite domain is to prove it mathematically.

\subsection{\methodnamefunctional}

In the functional approach to SLURP, we train a model $f_\theta \colon X \to Y$ by alternating two steps: (i) \textit{constraint projection} and (ii) \textit{supervised update}.
During constraint projection, we consider the function encoded by the current model $f_\theta$ and find the closest function $\hat f\colon X \to Y$ satisfying the rigid proportion constraints.
Typically, this means solving
\[
\min_{\hat f}\ \sum_{x \in X} d \! \left(f_\theta(x), \hat f(x) \right),
\]
where $d$ is a distance function on $Y$.
Then, in the supervised update, the pseudo-labels $\hat y = \hat f(x)$ are used as targets to run multiple SGD updates on $f_\theta$; in practice, we run one full epoch at every supervised update, setting aside a validation subset that is not used for gradient updates.
Note that the validation set undergoes the constraint projection step.

This procedure is repeated until $f_\theta$ satisfies the rigid proportion constraints (so, the projection step does nothing) or a maximum number of iterations is reached and the run is discarded.
It can be viewed as a self-training optimization resembling the classical expectation-maximization algorithm.

The stopping criterion is strict: a training run is considered successful only if the rigid proportion constraints are perfectly satisfied.
In particular, the model is required to generalize to the validation set in order to fit the constraints, thus accepting only sufficiently well-behaved functions.
The size of the validation set is effectively a hyperparameter requiring calibration: enlarging the validation set decreases the probability of convergence, while empirically filtering out less interpretable runs.

After convergence, the model $f_\theta$ encodes a function $X \to Y$; in our experiments we are often able to manually distill this function into a symbolic expression.
In general, distillation could be automated via symbolic regression or program synthesis
(optionally after ensembling multiple runs).

\subsection{\methodnamesymbolic}
\label{sec:mapseek-symbolic}

The symbolic approach begins by defining an alphabet $\Sigma$ of primitive symbols from which we construct a formal language $\L \subseteq \Sigma^*$ of syntactically valid formulas.
Each formula $\varphi \in \L$ defines a function $f_\varphi: X \to Y$ through its explicit algorithmic description, ensuring that functions expressible within $\L$ are inherently simple by virtue of their compact symbolic representation.

We define a distance $\delta:\L\to\R_{\geq 0}$ that quantifies how much $f_\varphi$ violates the rigid proportion constraints.
Our aim is to discover formulas $\varphi\in\L$ that achieve $\delta(\varphi)=0$, i.e.\ satisfying the constraints.
While any such formula technically solves the problem, we prefer shorter expressions, that are more likely to reveal mathematical structure and admit formal verification on infinite domains.

The search proceeds via an iterative population-based search, typically through the deep cross-entropy method using a generative sequence model $g_\theta$. The process alternates between (i) a \emph{generation phase}, where a batch of candidate formulas is proposed, and (ii) an \emph{evaluation phase}, where the distance $\delta$ is computed for each candidate.
The information from high-performing candidates—those with the lowest distance values—is aggregated to bias the generation mechanism in the subsequent iteration toward more promising regions of the search space. The loop continues until a formula $\varphi\in\L$ achieving $\delta(\varphi)=0$ is discovered, returning an exact symbolic solution.

\section{\texorpdfstring{$q,t$-Combinatorics Setup}{q,t-Combinatorics Setup}}
\label{sec:qtcombinatorics}

Within algebraic combinatorics, the area of $q,t$-com\-bi\-na\-to\-rics studies families of bivariate polynomials $F(q,t)$ with nonnegative integer coefficients arising in a variety of mathematical domains such as symmetric functions, representation theory, algebraic geometry, knot theory, probability, and integrable systems \cite{HaglundBook,ThetaOperators}.
Two recurring problems motivate much of the subject.

\textbf{Type 1 problem (combinatorial interpretation):}
Given $F(q,t)$, find a set $X$ of combinatorial objects and two functions $\stat_1,\stat_2\colon X\to\N$ (called \emph{statistics}) such that
\begin{equation}
     F(q, t) = \sum_{x \in X} q^{\stat_1(x)}\, t^{\stat_2(x)}.
\end{equation}

\textbf{Type 2 problem (exchanging bijection).}
Given a combinatorial interpretation, prove the symmetry $F(q,t)=F(t,q)$ by constructing a bijection $\psi\colon X\to X$ that exchanges the two statistics:
\[
    \big(\stat_1(\psi(x)),\stat_2(\psi(x))\big)
  \;=\;
  \big(\stat_2(x),\stat_1(x)\big).
\]

Both types of problems are typically difficult. For example, it took 8 years to find a combinatorial interpretation for the famous $q,t$-Catalan polynomials \cite{qtCatalan}, whereas a bijective proof of their symmetry remains open.
These problems provide a rich source of interesting SLURP instances: one searches for simple statistics and explicit bijections under exact distributional constraints.

\subsection{Narayana Polynomials}

In this paper we tackle the aforementioned problems on the $q,t$-Narayana polynomials $N_{n,k}(q,t)$, indexed by integers $n\ge k\ge 1$, arising from the diagonal coinvariants of the symmetric group \cite{qtNarayana,DeltaConjecture}.
They refine the classical Narayana numbers $N_{n,k}(1,1)$, which count many families of objects, including \textit{parallelogram polyominoes} and \textit{noncrossing partitions}.

In \citep{qtNarayana}, a combinatorial interpretation is given in terms of polyominoes (see Appendix~\ref{sec:qt-appendix}), and the symmetry
$N_{n,k}(q,t)=N_{n,k}(t,q)$ is proved by algebraic methods.
A bijective proof is still missing in general:
\begin{openproblem}
    Find an exchanging bijection proof of the symmetry $N_{n, k}(q, t) = N_{n, k}(t, q)$.
    \label{prob:narayana-bijection}
\end{openproblem}
In this work, we solve \Cref{prob:narayana-bijection} for the first interesting case $k=3$ (the cases $k \leq 2$ are trivial) via new interpretations of $N_{n,3}(q,t)$ in terms of noncrossing partitions.

\subsection{Noncrossing partitions}

A partition of $[n]=\{1,\dots,n\}$ is a collection of nonempty, pairwise disjoint subsets (called blocks) whose union is $[n]$.
For example, $\{\{1,4\},\{2,3,6\},\{5\}\}$ is a partition of $[6]$, and we write it more compactly as $14 \mid 236$, omitting singleton blocks.
A partition is \emph{noncrossing} if, when numbers are placed on a circle in increasing order and blocks are drawn by convex hulls, these hulls do not intersect.
Noncrossing partitions are counted by the Narayana numbers: the set $\NC(n, k)$ of noncrossing partitions with $n-k+1$ blocks has size $N_{n, k}(1, 1)$.

\begin{figure}[t]
    \centering
    \begin{tikzpicture}[scale=0.82, every node/.style={font=\scriptsize}]
    \def\r{0.55}

    \newcommand{\basefour}{
      \draw (0,0) circle (\r);

      \coordinate (v1) at ( 135:\r);
      \coordinate (v2) at (  45:\r);
      \coordinate (v3) at ( -45:\r);
      \coordinate (v4) at (-135:\r);

      \fill (v1) circle (1.6pt);
      \fill (v2) circle (1.6pt);
      \fill (v3) circle (1.6pt);
      \fill (v4) circle (1.6pt);

      \node[anchor=south east] at ($(0,0)!0.9!(v1)$) {$1$};
      \node[anchor=south west] at ($(0,0)!0.9!(v2)$) {$2$};
      \node[anchor=north west] at ($(0,0)!0.9!(v3)$) {$3$};
      \node[anchor=north east] at ($(0,0)!0.9!(v4)$) {$4$};
    }

    \begin{scope}[xshift=0.0cm]
      \basefour
      \draw[fill=black!20] (v1)--(v2)--(v3)--cycle;
    \end{scope}

    \begin{scope}[xshift=2cm]
      \basefour
      \draw[fill=black!20] (v1)--(v2)--(v4)--cycle;
    \end{scope}

    \begin{scope}[xshift=4cm]
      \basefour
      \draw[fill=black!20] (v1)--(v3)--(v4)--cycle;
    \end{scope}

    \begin{scope}[xshift=6cm]
      \basefour
      \draw[fill=black!20] (v2)--(v3)--(v4)--cycle;
    \end{scope}

    \begin{scope}[xshift=8cm]
      \basefour
      \draw (v1)--(v2);
      \draw (v3)--(v4);
    \end{scope}

    \begin{scope}[xshift=10cm]
      \basefour
      \draw (v1)--(v4);
      \draw (v2)--(v3);
    \end{scope}

    \end{tikzpicture}
    \caption{
        Noncrossing partitions $\NC(4, 3)$.
        The corresponding $q,t$-Narayana is $q^2 + qt + t^2 + q + t + 1$.
    }
    \label{fig:ncpartitions-n4-k3}
\end{figure}
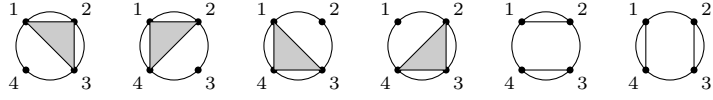

For $k=3$, we obtain partitions with $n-2$ blocks (\Cref{fig:ncpartitions-n4-k3}), which take the following two forms:
\begin{itemize}[itemsep=2pt, topsep=0pt, parsep=0pt, partopsep=0pt]
    \item $ab \mid cd$ with $a<b$, $a<c$, and $c<d$;
    \item $abc$ with $a<b<c$.
\end{itemize}

We introduce a natural statistic on noncrossing partitions, to be used as a starting point to search for new combinatorial interpretations of Narayana polynomials. This statistic, called \textit{skip}, counts how many numbers are skipped inside each block of a partition, and is defined as follows:
\begin{equation}
    \skipstat(\pi) = \sum_{\beta \in \pi} (\max \beta - \min\beta - |\beta| + 1).
    \label{eq:skip}
\end{equation}

\section{The Skip-Pairing Benchmark}
\label{sec:skip-testbed}

As a testbed for combinatorial function discovery within the SLURP framework, we propose the concrete problem of constructing simple statistics on noncrossing partitions that, together with the skip statistic \eqref{eq:skip}, provide a combinatorial interpretation of the Narayana polynomials~$N_{n, k}(q, t)$.

Practically, this means that noncrossing partitions are grouped by the value $u = \skipstat(\pi)$; within each group, the number of partitions $\pi$ with $\stat(\pi) = v$ must equal the coefficient of $q^u t^v$ in $N_{n, k}(q, t)$.

To increase supervision, we optionally refine the constraints via a natural filtration of noncrossing partitions:
for a partition $\pi$, let $m(\pi)$ be the largest element that belongs to a non-singleton block;
group partitions by the pair $(m,u)$ where $m=m(\pi)$ and $u=\skipstat(\pi)$.
In this refined setting, the expected distribution for each group is prescribed by the incremental polynomials
$N_{m,k}(q,t) - N_{m-1,k}(q,t)$, which have nonnegative coefficients by \Cref{prop:incremental-narayana-positive}.

To foster research on combinatorial function discovery, we release code to generate the dataset for the skip-pairing problem for any $n \geq k \geq 1$, both in the unrefined and refined setting.
In \Cref{sec:experiments}, we use the refined $N_{14, 3}$ dataset, consisting of $2366$ partitions grouped into $144$ bags, to discover multiple interesting statistics paired with skip for $k=3$.

\section{Experiments and Results}
\label{sec:experiments}

\subsection{Statistic Discovery with \methodnamefunctional}
\label{sec:experiments-functional}

\paragraph{Model Architecture.}
We run the \methodnamefunctional pipeline using a Transformer encoder with $8$ layers, $4$ attention heads, hidden dimension $128$, and feedforward dimension $256$.
The input to the Transformer is a fixed sequence of $n$ tokens, encoded by taking the real numbers $i/(n-1)$ for $i=0, \dots, n-1$ and passing them through a linear embedding layer (with bias).

Each partition $\pi$ is represented by its block‑adjacency matrix $M\in\{0,1\}^{n\times n}$ where $M_{ij}=1$ if $i$ and $j$ lie in the same block.
The partition matrix $M$ enters the model as an attention mask: in alternating layers, self‑attention is restricted to entries with $M_{ij}=1$ (the final layer is unmasked).
The model outputs the attention logits from the last layer, averaged over heads and passed through a sigmoid, yielding a score matrix $S\in [0, 1]^{n\times n}$.
The sum $\sum_{i,j} S_{ij}$ is the model's output, kept as a floating point number for the constraint projection step.

\paragraph{Training procedure.}
Within each family (same $(m, u)$), we sort partitions by the current model's output and assign pseudo-labels matching the required multiplicities from the rigid proportion constraints.
If $v$ is the pseudo-label assigned to a partition, we construct a binary $n \times n$ target matrix by taking the top-$v$ entries of $S$.
The model is trained with a binary cross-entropy loss between $S$ and this target mask.
We use the Adam optimizer with learning rate $10^{-4}$ and batch size $32$.
In each supervised update step, we run one full training epoch on the current pseudo‑labels.

\paragraph{Initial experiments.}
We run \methodnamefunctional on the refined $N_{14, 3}$ skip-pairing dataset with a budget of $1000$ iterations and a $1\%$ validation split.
However, we find the outputs of successful runs difficult to distill into simple formulas.
Therefore, in subsequent experiments we add extra bias (in the form of warm start) or homogeneity constraints.
Note that the ability to incorporate such information into the search allows mathematicians to inject expert knowledge and intuition about the potential structure or properties of the solution. Often a candidate statistic is available but does not satisfy all constraints, and an automated procedure to modify into a working statistic can be extremely helpful.

\paragraph{Experiments with warm start.}
In this setup, we warm-start each run with $20$ epochs of supervised pretraining on the following heuristic (human-defined) candidate, which matches the constraints for $n \leq 4$ but not for $n \geq 5$:
\begin{equation*}
    \sum_{\beta \in \pi} \left( \textrm{sum}(\beta) - \max(\beta) \right) - \textstyle\frac{k(k-1)}{2}.
\end{equation*}
Specifically, the supervised pretraining target for the score matrix $S$ is the $n \times n$ binary matrix $T$ where $T_{ij} = 1$ if $i$ is not the maximum of its block and $1 \le j \le i - q(i)$, with $q(i)$ the number of non‑largest elements $\le i$.

We then ran \methodnamefunctional for up to $300$ iterations with a $10\%$ validation split.
Out of $76$ successful runs, $61$ runs converged to exactly the same predictions on all $2366$ partitions.
We name the resulting statistic \textit{mag} (for ``magnitude'')
and manually deduce the following formula:
\begin{equation}
    \begin{split}
        \magnitude(ab \mid cd) &= a + \max(b, c) - 2 - \left\lceil \tfrac{b - a}{2} \right\rceil \\
        \magnitude(abd) &= a + b - 2 - \left\lceil \tfrac{b - a}{2} \right\rceil
    \end{split}
    \label{eq:magnitude}
\end{equation}
(the second formula being the limit case $b=c$ of the first formula).
Interestingly, this resembles the $t$-statistic of \citep{gorsky2020generalized}, which pairs with the area statistic on certain combinatorial objects, yielding a combinatorial interpretation of other $q,t$-polynomials.

\paragraph{Experiments with homogeneity constraints.}
As an alternative to warm start, which requires a heuristic candidate, we add supervision to \methodnamefunctional by forcing homogeneity (i.e., translation-equivariance).
We say that a statistic is $h$-homogeneous if, when adding $1$ to all elements of each non-singleton block, the statistic increases by $h$.
For example, the skip is $0$-homogeneous and the mag is $2$-homogeneous.
It is easy to see that $h$-homogeneous statistics paired with skip only exist for $h=1,2$ (either of these constraints still leaves huge flexibility).

We enforce this in the projection step, assigning pseudo-labels in increasing value of $m$ while greedily propagating homogeneity.

We ran this setup using a $5\%$ validation split and a budget of $1000$ iterations, obtaining $14$ successful runs out of $50$ for $h=1$ and $18$ out of $50$ for $h=2$.

For $h=1$, we had $4$ runs converging to exactly the same predictions, described by the following formula (we name this the \textit{leap} statistic):
\begin{equation}
    \begin{split}
        \leap(ab \mid cd) &= c - 2 + \max(0, \, c-b) \\
        \leap(abd) &= b - 2
    \end{split}
    \label{eq:leap}
\end{equation}
(the second formula is the limit case $b=c$ of the first one).
Out of the remaining $10$ successful runs, $5$ differed from leap by less than $0.5\%$, and the other $5$ were much farther from leap and each other ($\ge5\%$ disagreement).

For $h=2$, interpretation was slightly less straightforward.
We clustered the $18$ successful runs, treating runs as $2366$-dimensional prediction vectors, using agglomerative hierarchical clustering with average linkage (\Cref{fig:functional-pca-h2}).
The highest silhouette score was obtained with $3$ clusters of sizes $11$, $6$, and $1$.
The medoid of the second cluster yields exactly the mag statistic \eqref{eq:magnitude}; the medoid of the first cluster yields a modification of mag, with a $\pm 1$ correction based on the parity of $b-a$ when $b < c$.

Overall, for both $h=1,2$ we found two ostensibly simple statistics (leap and mag), despite \methodnamefunctional having no explicit bias toward short symbolic descriptions.

\begin{figure}[t]
  \begin{center}
    \begin{subfigure}[t]{0.495\linewidth}
      \centering
      \includegraphics[width=0.85\linewidth]{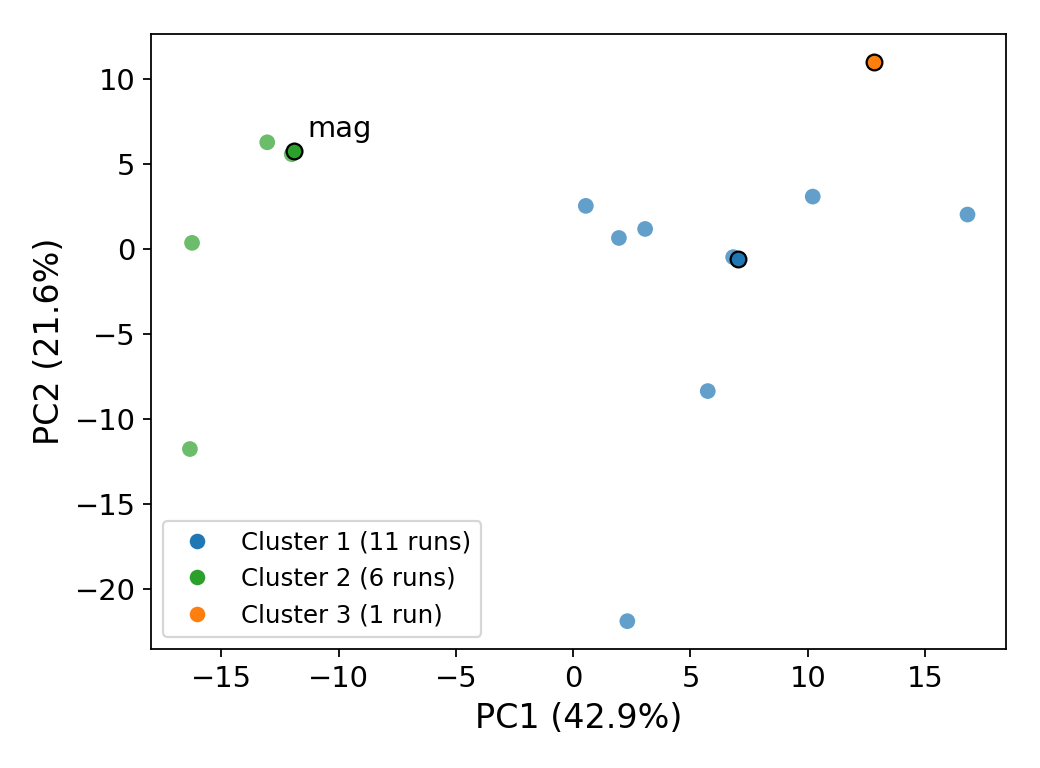}
    \end{subfigure}
    \hfill
    \begin{subfigure}[t]{0.495\linewidth}
      \includegraphics[width=\linewidth]{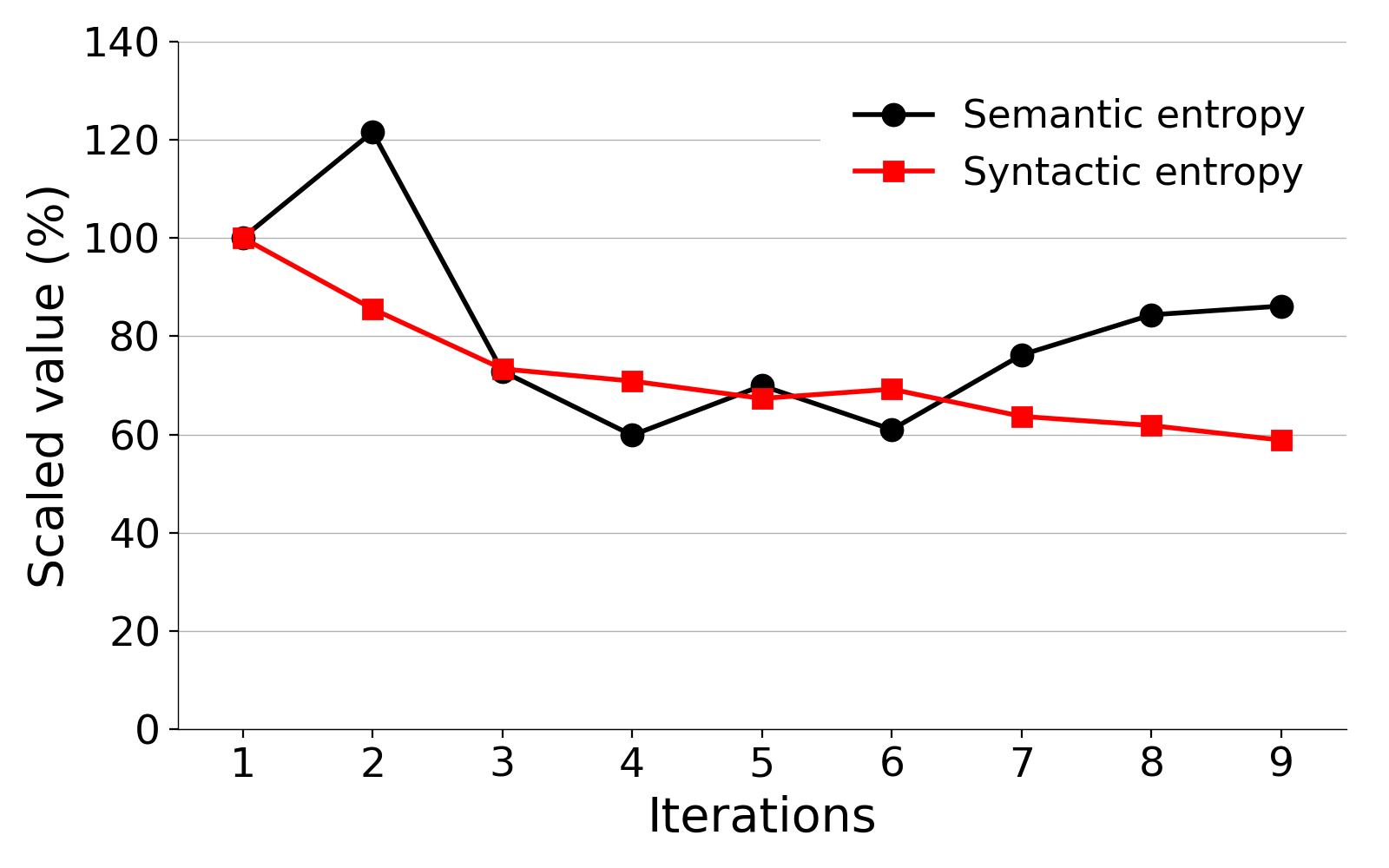}
    \end{subfigure}
    \caption{
      Left: PCA scatter plot of the \methodnamefunctional successful runs (refined $N_{14, 3}$ skip-pairing dataset, $h=2$). Cluster medoids are highlighted. The explained variance of the principal components PC1 and PC2 is shown on the axes.
      Right: semantic and syntactic entropy of a representative run. Values are scaled with respect to the first iteration.
    }
    \label{fig:functional-pca-h2}
    \label{fig:semantic-collapse}
  \end{center}
\end{figure}

\paragraph{Ablation study on validation set size.}
A peculiarity of \methodnamefunctional is that the size of the validation set is a hyperparameter trading off success rate and simplicity.
\Cref{tab:valsize-ablation} shows its effect in the $h=2$ setup. %

As the validation split increases, the success rate drops.
The number of runs matching mag exactly stays roughly constant, so a larger split appears to filter out less interpretable solutions.
Notably, at $10\%$ validation, the two successful runs are nearly identical, and one coincides with mag.

\begin{table}[t]
    \centering
    \caption[Ablation on validation set size in MapSeek-Functional]{Ablation on validation set size in \methodnamefunctional (refined $N_{14,3}$ skip-pairing dataset, $h=2$, with $50$ runs per setting).}
    \label{tab:valsize-ablation}
    \setlength{\tabcolsep}{3.2pt}
    \begin{tabular}{ccccc}
    \toprule
    Val & \makecell{\# Success} & \makecell{Med.\ disagree\\ between runs} & \makecell{Med.\ disagree\\ with mag} & \# Exact mag \\
    \midrule
    $0\%$  & $49$ & $14.8\%$ & $18.6\%$ & $1$ \\
    $1\%$  & $42$ & $11.7\%$ & $18.6\%$ & $2$ \\
    $2\%$  & $27$ & $15.3\%$ & $15.3\%$ & $3$ \\
    $5\%$  & $18$ & $16.9\%$ & $16.8\%$ & $1$ \\
    $10\%$ & $2$ & $0.13\%$ & $0.06\%$ & $1$ \\
    \bottomrule
    \end{tabular}
\end{table}

\subsection{Statistic Discovery with \methodnamesymbolic}
\label{sec:experiments-symbolic}

\paragraph{Formulas and Evaluation.}
We run \methodnamesymbolic using the alphabet
\[
    \Sigma=\{a,b,c,d,i,0,1,2,3,+,-,\%, (,),\vee,\wedge,\neg,>,<,=\}.
\]
A valid formula $\varphi\in\Sigma^*$ defines an integer-valued function $f_\varphi$ on noncrossing partitions with $k=3$ by evaluating $\varphi$ on a $4$-tuple $(a,b,c,d)$ (for a $3$-block partition $abd$ we use $(a,b,b,d)$).
Logical operators behave like Python short-circuit operators and can be used to express conditional integer-valued expressions; comparisons produce $0/1$.
Given a syntactically valid formula $\varphi$, we can treat it as defining a function by summing the formula evaluation over the additional variable $i$:
\[
    f_\varphi(\pi) = \sum_{i=1}^{n} \operatorname{eval}\!\big(\varphi; a, b, c, d, i \big).
\]

\paragraph{Distance.}
We define a distance function $\delta \colon \L \to \R_{\geq 0}$ measuring how far the formula $\varphi$ deviates from satisfying the constraints.
For each bag $B_j$, the constraints specify a target distribution of output values $D_j$, prescribed by the refined $N_{14,3}$ skip-pairing dataset.
Let $D_j^\varphi = \{ f_\varphi(x) \mid x \in B_j \}$ be the distribution of values produced by formula $\varphi$ on the bag $B_j$.
We define $\operatorname{count}(y, S)$ as the multiplicity of the value $y$ in the distribution $S$, and $\delta$ as a sum of $L^1$ distances between these distributions:
\begin{equation}
    \delta(\varphi) = \sum_{j} \sum_{y \in Y} \left| \operatorname{count}(y, D_j) - \operatorname{count}(y, D_j^\varphi) \right|.
    \label{eq:naive_distance}
\end{equation}

\paragraph{Initial formulas.}

We initialize each run by sampling one million unique well-formed formulas, constructed via a recursive grammar that composes arithmetic, comparison, and logical operations.
We restrict candidates to the form $(\text{expression}) \,\%\, 3$;
this way, each formula evaluation lies in $\{0,1,2\}$ and the resulting value $f_\varphi(\pi)$ lies between $0$ and $2n$ (the actual maximum value is $2n-6$).
To maximize search coverage, different runs employ one of $20$ distinct hyperparameter configurations, varying recursion depth ($2$--$4$), operator probabilities, and length ($5$--$60$ characters).

\paragraph{Genetic algorithm (GA) baseline.}
We implemented a GA baseline that optimizes formulas directly using the distance $\delta$ as fitness.
The GA consistently finds zero-distance formulas yielding the leap statistic \eqref{eq:leap}, but with extreme formula bloat ($100+$ characters long); this behavior is consistent with the fitness-causes-bloat effect \cite{bloat1}.
We applied standard anti-bloat measures (e.g.\ depth/size limits) but could not obtain consistently compact formulas satisfying the constraints, in accordance with empirical studies that highlight the limits of such regularizers \cite{bloat2}.
Therefore, we moved to a controlled symbolic approach using the deep cross-entropy method (CEM) with a generative sequence model.

\paragraph{Model architecture.}
We run the deep CEM pipeline using a Transformer decoder with $4$ layers, $8$ attention heads, hidden dimension $256$, and feedforward dimension $256$. We employ a character-level tokenizer.
\paragraph{Training Procedure.} Each run begins by evaluating the distance $\delta(\varphi)$ for all $1$ million randomly generated formulas in the initial population. We select the top $5\%$ formulas with the lowest distance values to form the elite subset, from which we reserve $1000$ for validation and use the remaining as training data for the Transformer model.
The model is trained from scratch using standard cross-entropy loss for $60$ epochs with the AdamW optimizer, learning rate $5 \times 10^{-5}$, weight decay $0.001$, and batch size $2048$.
We select the checkpoint with the lowest validation loss.
We then sample $1$ million new formulas from this trained model, evaluate their distances, extract a new elite subset, and repeat the process.
Crucially, the Transformer is retrained from scratch at each iteration to maintain exploration throughout the search.

\paragraph{Skip-pairing experiments.} We ran the deep CEM pipeline %
and consistently obtained successful runs with $7$ of the $20$ hyperparameter configurations defining the initial population.
Remarkably, the runs always converged to the leap statistic in $4$ of these configurations, each time discovering syntactically different formulas that represent the same function.
In the remaining $3$ successful configurations, the following new statistic is discovered:
\begin{equation}
    \begin{split}
        \sk(ab \mid cd) &= c - 2 +
        \begin{cases}
            a & \text{if } b < c \\
            0 & \text{otherwise}
        \end{cases} \\
        \sk(abd) &= b - 2.
    \end{split}
    \label{eq:skew}
\end{equation}

\paragraph{Training dynamics.}
To understand the training dynamics of the deep CEM approach, we monitor population diversity across iterations through the following metrics (\Cref{fig:semantic-collapse}): the entropy of the distribution of unique functions computed by the generated formulas (\emph{semantic}), and the token-level entropy of the model (\emph{syntactic}).
Semantic entropy remains relatively stable, indicating that the population does not degenerate into a set of functionally identical clones.
In contrast, syntactic entropy exhibits a steady, monotonic decline.
This suggests that the model learns to exploit the structural properties of successful formulas, gradually focusing the search space while retaining exploratory variance.
This is unlike standard RL on LMs, which suffers from entropy collapse \cite{entropy-collapse}.

\paragraph{Formula notation and stability.}
The symbolic representation of formulas significantly impacts training stability. We compared three notations of varying compactness: reverse polish notation (most compact, no parentheses), standard infix notation, and infix with explicit whitespace.
We find that compactness inversely correlates with stability: RPN runs degraded rapidly, producing illegible formulas after few iterations, while whitespace-enriched notation gave the best results.
We hypothesize that explicit structural delimiters (parentheses and spaces) provide cues that help the Transformer preserve coherent formula structure during generation. Without these, the model tends to generate syntactically valid but redundant expressions, such as
\[
    (((c>i) \,\%\, i \,\%\, 23)\wedge((b>i)\wedge2\wedge0\,\%\,30+10\vee12))\,\%\,3.
\]

\paragraph{Wasserstein: $L^1$ vs trivial metric.}
We investigated whether replacing the distance $\delta$ with the sorted $L^1$ distance would improve convergence.
From an optimal transport perspective, these two objectives represent the Wasserstein distance equipped with different ground metrics: the trivial metric, where the cost is $0$ for identical points and $1$ otherwise, and the $1$-dimensional $L^1$ metric.
The trivial metric treats statistic discovery as a classification problem, viewing all output values as equidistant classes; conversely, the $L^1$ metric interprets the task as regression by rewarding numerical proximity.
\methodnamesymbolic failed to converge in all runs using the sorted $L^1$ distance, suggesting that the problem is better regarded as a classification task: the constraints demand exact symbolic matching, and the magnitude of numerical violation is not a useful learning signal.

\subsection{Extending the Atlas}
\label{sec:extending-the-atlas}

More combinatorial interpretations of $q,t$-Narayana polynomials $N_{n, 3}(q, t)$ emerge by running \methodnamefunctional and \methodnamesymbolic on statistics other than skip.
\Cref{fig:stat-atlas} shows the most notable statistics and connections between them found by the two methods.

Although both methods used a finite training dataset, we are able to mathematically verify most discovered pairings:
for every edge in \Cref{fig:stat-atlas}, we prove the joint distribution matches $N_{n, 3}(q, t)$ for all $n$.
Proofs are formalized in Lean~4 and available as supplementary material (see Appendix~\ref{sec:lean-appendix}).
This validates our claim that infinite SLURP problems can be effectively solved by applying ML to finite subproblems.

A large number of emerging connections involve the \textit{area} statistic, which (like the \textit{bounce} statistic) is the translation of a known statistic on polyominoes (see Appendix~\ref{sec:qt-appendix}).
This is notable as the area is arguably the simplest known statistic on polyominoes related to $q,t$-Narayana polynomials, thus suggesting that both methods are biased toward human-designed statistics.

\section{ML-Assisted Mathematical Contributions}
\label{sec:bijection}

In addition to discovering several interpretations of $N_{n, 3}$ in Section~\ref{sec:experiments}, the machine-found statistics also lead to two human mathematical contributions.
First, we prove an exchanging bijection for $(\skipstat,\leap)$, hence giving the first combinatorial proof of the symmetry $N_{n, 3}(q, t) = N_{n, 3}(t, q)$.
For $a < b \leq c < d$ the four elements in non-singleton blocks, with $b = c$ in the size-$3$ case, the map reflects $b$ in $[a, c]$ and then reflects $a, b, c$ in $[0, d]$ (Appendix~\ref{sec:qt-appendix}, \Cref{fig:bijection-n4-k3}).
In a formula:
\begin{equation*}
   (a,\, b,\, c,\, d) \mapsto (d-c, \; b+d-a-c, \; d-a, \; d).
\end{equation*}

Moreover, by extending the flipped version of \textit{skew} (its composition with the flip map $i \mapsto n+1-i$), we obtain a new statistic, called \textit{mingarc} (min gap arc), which pairs with skip to give the first combinatorial interpretation of $N_{n,k}(q,t)$ using noncrossing partitions for all $k$ (see Appendix~\ref{sec:mingarc-appendix}).

We formalize that the above bijection exchanges skip and leap; combined with the Lean proof that $(\skipstat, \leap)$ is distributed according to $N_{n, 3}(q, t)$ for all $n$, this gives a fully verifiable solution to \Cref{prob:narayana-bijection} for $k=3$.
We also formalize the skip-mingarc pairing, providing the first combinatorial interpretation of $N_{n, k}(q, t)$ based on noncrossing partitions for all $k$.

\section{Conclusions and Future Directions}

This paper shows how ML can be effectively applied to mathematical discovery problems where rigid constraints must be exactly satisfied and simplicity is a key requirement.

Several directions appear immediately.
The main mathematical challenge lies in more difficult $q,t$-combinatorics problems, such as discovering new $N_{n, k}(q,t)$ interpretations for $k \geq 4$.
While the statistics and exchanging bijection we obtain for $k=3$ can be written in closed forms, we expect many natural solutions for larger $k$ to be inherently algorithmic (e.g, the bounce statistic, defined in Appendix~\ref{sec:qt-appendix}).

On the methodological side, our current $k=4$ experiments highlight concrete failure modes that must be resolved: \methodnamefunctional converges to solutions seemingly difficult to interpret; \methodnamesymbolic is unable to reach zero distance from exact constraint match.

More broadly, it remains to understand how SLURP methods explore the space of simple solutions, and whether all solutions a human mathematician would consider meaningful can in fact be found automatically.
It would also be beneficial to quantify simplicity and generalization through principled metrics for SLURP problems.

\IfFileExists{./\jobname.bbl}{
  \bibliographystyle{plainnat}
  \bibliography{\paperdir/bibliography}
}{
  \IfFileExists{\paperdir/\jobname.bbl}{

  }{
    \bibliographystyle{plainnat}
    \bibliography{\paperdir/bibliography}
  }
}

\newpage
\appendix
\section{\texorpdfstring{$q,t$-Combinatorics Background \& Narayana Models}{q,t-Combinatorics Background and Narayana Models}}
\label{sec:qt-appendix}
Triggered by the discovery of the famous Macdonald polynomials \cite{MacdonaldOriginalSLC}, $q,t$-combinatorics studies families of bivariate polynomials $F(q,t)$ with nonnegative integer coefficients arising in a variety of mathematical domains.

Historically, the first examples that attracted an intense research activity occurred in the study of the diagonal coinvariants of the symmetric group \cite{nFactorial}.
The bivariate polynomials arising from these algebraic objects keep track of the multiplicities of their ``building blocks'' (irreducible representations): computing these multiplicities is the prototypical problem of representation theory. The $q,t$-Narayana polynomials $N_{n,k}(q,t)$ are a notable example.

Interest in $q,t$-combinatorics flourished over the past 25 years.
Consequently, a formidable number of new formulas, conjectures, open problems, and interactions with other disciplines have been discovered.
Among these new problems, a prominent role is played by the Type 1 and Type 2 problems mentioned in Section~\ref{sec:qtcombinatorics} \cite{HaglundBook,ThetaOperators}.
Our hope is that the methods presented in this paper will lead to breakthroughs in this fascinating area.

The Type 2 problem of proving combinatorially the $q,t$-symmetry $N_{n,k}(q,t)=N_{n,k}(t,q)$ is one of the outstanding open problems of the theory. A combinatorial interpretation of $N_{n,k}(q,t)$ in terms of \emph{polyominoes}, together with an algebraic proof of its $q,t$-symmetry, first appeared in \citep{qtNarayana}.
The relation between $q,t$-Narayana, polyominoes, and noncrossing partitions is described in the next sections.

\subsection{Polyominoes}

A \emph{parallelogram polyomino} of size $m \times n$ is a pair of lattice paths from $(0,0)$ to $(m,n)$, using north and east steps, that meet only in the extremal points (\Cref{fig:polyomino-area-bounce}).
The set of all polyominoes of size $m \times n$ is denoted by $\PP(m,n)$.

\definecolor{oiBlue}{RGB}{0,74,138}

\tikzset{
  idxlab/.style={font=\tiny, text=oiBlue},
  toplab/.style={font=\tiny, text=purple},
  bouncelab/.style={font=\scriptsize, text=red},
  labbg/.style={fill=white, fill opacity=0.75, text opacity=1, inner sep=1pt}
}

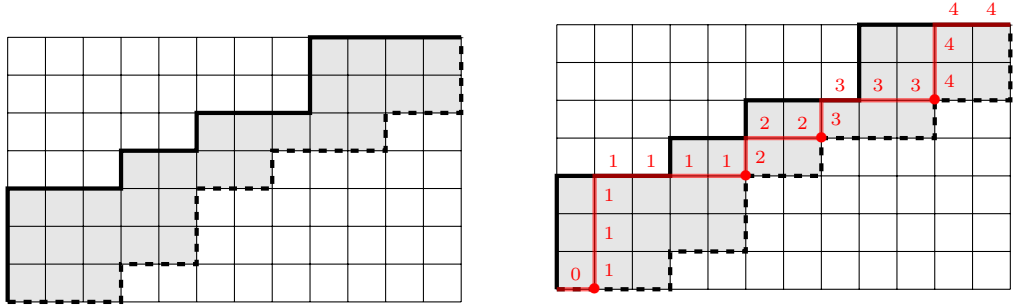
\begin{figure}[!ht]
  \centering
  \begin{minipage}[t]{0.48\textwidth}
    \centering
    \begin{tikzpicture}[scale=0.5]
    \draw[black, thin] (0,0) grid (12,7);

    \filldraw[black, opacity=0.1] (0,0) -- (3,0) -- (3,1) -- (5,1) -- (5,3) -- (7,3) -- (7,4) -- (10,4) -- (10,5) -- (12,5) -- (12,7) -- (8,7) -- (8,5) -- (5,5) -- (5,4) -- (3,4) -- (3,3) -- (0,3) -- cycle;

    \draw[dashed, line width=1.6pt] (0,0) -- (3,0) -- (3,1) -- (5,1) -- (5,3) -- (7,3) -- (7,4) -- (10,4) -- (10,5) -- (12,5) -- (12,7);

    \draw[line width=1.6pt] (0,0) -- (0,3) -- (3,3) -- (3,4) -- (5,4) -- (5,5) -- (8,5) -- (8,7) -- (12,7);

    \end{tikzpicture}
  \end{minipage}
  \hfill
  \begin{minipage}[t]{0.48\textwidth}
    \centering
    \begin{tikzpicture}[scale=0.5]
    \draw[black, thin] (0,0) grid (12,7);

    \filldraw[black, opacity=0.1] (0,0) -- (3,0) -- (3,1) -- (5,1) -- (5,3) -- (7,3) -- (7,4) -- (10,4) -- (10,5) -- (12,5) -- (12,7) -- (8,7) -- (8,5) -- (5,5) -- (5,4) -- (3,4) -- (3,3) -- (0,3) -- cycle;

    \draw[dashed, line width=1.6pt] (0,0) -- (3,0) -- (3,1) -- (5,1) -- (5,3) -- (7,3) -- (7,4) -- (10,4) -- (10,5) -- (12,5) -- (12,7);
    \draw[line width=1.6pt] (0,0) -- (0,3) -- (3,3) -- (3,4) -- (5,4) -- (5,5) -- (8,5) -- (8,7) -- (12,7);

    \draw[red,  line width=1.6pt, opacity=0.6] (0,0) -- (1,0) -- (1,3) -- (5,3) -- (5,4) -- (7,4) -- (7,5) -- (10,5) -- (10,7) -- (12,7);
    \node[red] at (1,0) {$\bullet$};
    \node[red] at (5,3) {$\bullet$};
    \node[red] at (7,4) {$\bullet$};
    \node[red] at (10,5) {$\bullet$};
      \begin{scope}[bouncelab]
        \node[above] at (0.5,0) {$0$};
    	\node[right] at (1,0.5) {$1$};
    	\node[right] at (1,1.5) {$1$};
    	\node[right] at (1,2.5) {$1$};
    	\node[above] at (1.5,3) {$1$};
    	\node[above] at (2.5,3) {$1$};
    	\node[above] at (3.5,3) {$1$};
    	\node[above] at (4.5,3) {$1$};
    	\node[right] at (5,3.5) {$2$};
    	\node[above] at (5.5,4) {$2$};
    	\node[above] at (6.5,4) {$2$};
    	\node[right] at (7,4.5) {$3$};
    	\node[above] at (7.5,5) {$3$};
    	\node[above] at (8.5,5) {$3$};
    	\node[above] at (9.5,5) {$3$};
    	\node[right] at (10,5.5) {$4$};
    	\node[right] at (10,6.5) {$4$};
    	\node[above] at (10.5,7) {$4$};
    	\node[above] at (11.5,7) {$4$};
      \end{scope}
    \end{tikzpicture}
  \end{minipage}

  \caption{A $12 \times 7$ polyomino with area $30 - 18 = 12$ (left) and its bounce path (right); the bounce statistic is equal to $41 - 18 = 23$.}
  \label{fig:polyomino-area-bounce}
\end{figure}

\begin{definition}
	The \emph{area} statistic on $\PP(m,n)$ is the number of whole squares between the two paths, normalized by subtracting $m+n-1$ (so that the minimum possible area is $0$).
\end{definition}

Parallelogram polyominoes have a \emph{bounce path}. To compute it, start from $(0, 0)$, draw a single east step, and then follow this algorithm: draw north steps until the path hits the polyomino's ceiling; then draw east steps until the path hits the polyomino's floor; repeat until it reaches $(m,n)$, as shown in \Cref{fig:polyomino-area-bounce}.

\begin{definition}
    The \emph{bounce} statistic on $\PP(m,n)$ is defined as the sum, over all steps of the bounce path, of the number of east-north turns preceding the step;
    normalize by subtracting $m+n-1$ (so that the minimum possible bounce is $0$).
    \label{def:bounce}
\end{definition}

A third statistic called \textit{dinv} was also introduced in \citep{qtNarayana}.
Both pairs $(\area, \bounce)$ and $(\dinv, \area)$ were proven to yield combinatorial interpretations of the $q,t$-Narayana polynomials:
\[ N_{n, k}(q,t) = \sum_{P\in \PP(n-k+1,k)}q^{\area(P)}t^{\bounce(P)} = \sum_{P\in \PP(n-k+1,k)}q^{\dinv(P)}t^{\area(P)}. \]
In addition, the symmetry $N_{n, k}(q, t) = N_{n, k}(t, q)$ was proved by algebraic methods; a combinatorial proof is still unknown.

\begin{figure}[ht]
  \centering
  \begin{tikzpicture}[scale=0.78, every node/.style={font=\scriptsize}]
    \begin{scope}
      \draw[gray!45, thin] (0,0) grid (4,3);
      \filldraw[fill=black!10, draw=none]
        (0,0) -- (1,0) -- (1,1) -- (2,1) -- (3,1) -- (3,2) -- (4,2) -- (4,3) --
        (1,3) -- (1,2) -- (0,2) -- cycle;
      \draw[dashed, line width=1.2pt] (0,0) -- (1,0) -- (1,1) -- (2,1) -- (3,1) -- (3,2) -- (4,2) -- (4,3);
      \draw[line width=1.2pt] (0,0) -- (0,1) -- (0,2) -- (1,2) -- (1,3) -- (2,3) -- (3,3) -- (4,3);
      \draw[red, line width=1.2pt, opacity=0.6] (0,0) -- (1,0) -- (1,2) -- (3,2) -- (3,3) -- (4,3);
      \node[red] at (1,0) {$\bullet$};
      \node[red] at (3,2) {$\bullet$};
      \begin{scope}[bouncelab]
        \node[above] at (0.5,0) {$0$};
        \node[right] at (1,0.5) {$1$};
        \node[right] at (1,1.5) {$1$};
        \node[above] at (1.5,2) {$1$};
        \node[above] at (2.5,2) {$1$};
        \node[right] at (3,2.5) {$2$};
        \node[above] at (3.5,3) {$2$};
      \end{scope}
      \node at (2,-0.55) {$P \in \PP(4,3)$};
    \end{scope}

    \draw[->, thick] (4.7,1.5) -- (6.3,1.5) node[midway, above] {$\iota$};

    \begin{scope}[xshift=7.1cm]
      \draw[gray!45, thin] (0,0) grid (4,4);
      \filldraw[fill=black!10, draw=none]
        (0,0) -- (1,0) -- (1,2) -- (2,2) -- (3,2) -- (3,3) -- (4,3) -- (4,4) --
        (1,4) -- (1,3) -- (0,3) -- cycle;
      \filldraw[fill=blue!25, draw=blue!50!black]
        (0,0) rectangle (1,1);
      \draw[dashed, line width=1.2pt] (0,0) -- (1,0) -- (1,1) -- (1,2) -- (2,2) -- (3,2) -- (3,3) -- (4,3) -- (4,4);
      \draw[line width=1.2pt] (0,0) -- (0,1) -- (0,2) -- (0,3) -- (1,3) -- (1,4) -- (2,4) -- (3,4) -- (4,4);
      \draw[red, line width=1.2pt, opacity=0.6] (0,0) -- (1,0) -- (1,3) -- (3,3) -- (3,4) -- (4,4);
      \node[red] at (1,0) {$\bullet$};
      \node[red] at (3,3) {$\bullet$};
      \begin{scope}[bouncelab]
        \node[above] at (0.5,0) {$0$};
        \node[right] at (1,0.5) {$1$};
        \node[right] at (1,1.5) {$1$};
        \node[right] at (1,2.5) {$1$};
        \node[above] at (1.5,3) {$1$};
        \node[above] at (2.5,3) {$1$};
        \node[right] at (3,3.5) {$2$};
        \node[above] at (3.5,4) {$2$};
      \end{scope}
      \node at (2,-0.55) {$\iota(P) \in \PP(4,4)$};
    \end{scope}
  \end{tikzpicture}
  \caption{The embedding used in \Cref{prop:incremental-narayana-positive}. A new bottom row is added, but only its leftmost cell is kept, so there is exactly one more square between the two paths. The red bounce path acquires one extra step in its first vertical run.}
  \label{fig:polyomino-growth-embedding}
\end{figure}
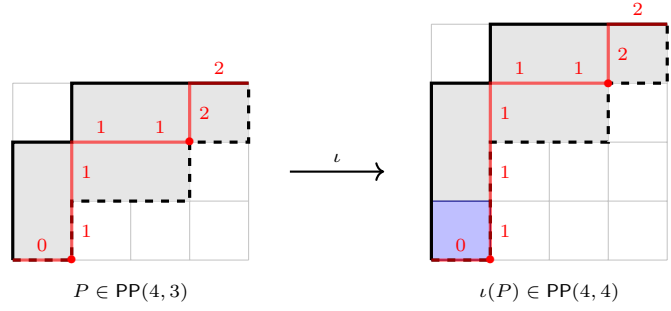

The refinement introduced in Section~\ref{sec:skip-testbed} is justified by the following monotonicity property of $q,t$-Narayana polynomials.

\begin{proposition}
    \label{prop:incremental-narayana-positive}
    For every fixed $k \ge 1$, each coefficient of $N_{m,k}(q,t)$ is nondecreasing in $m$.
    Equivalently, the incremental polynomial
    \[
        N_{m,k}(q,t)-N_{m-1,k}(q,t)
    \]
    has nonnegative coefficients.
\end{proposition}

\begin{proof}
    It is convenient to draw the model so that the rectangle side depending on $m$ is vertical; this is justified by the symmetry of the $q,t$-Narayana polynomials in the two rectangle directions \citep[Theorem~6.2]{qtNarayana}.
    In this picture, increasing $m$ by $1$ amounts to adding one new bottom row.
    Given a polyomino $P$, define $\iota(P)$ by adding such a bottom row and keeping only its leftmost cell, as in \Cref{fig:polyomino-growth-embedding}.
    Equivalently, if the upper path of $P$ is $U$ and the lower path is $EL$, then the upper and lower paths of $\iota(P)$ are $NU$ and $ENL$, respectively.
    This construction is clearly injective.

    The unnormalized area increases by exactly $1$, because $\iota(P)$ contains exactly one extra cell.
    But the normalization also increases by $1$, since we now subtract one more unit from the enclosing rectangle size.
    Therefore $\area(\iota(P))=\area(P)$.

    The bounce path of $\iota(P)$ is obtained from the bounce path of $P$ by inserting one extra step in its first vertical run.
    Hence the unnormalized bounce also increases by $1$.
    Again, this is exactly canceled by the larger normalization term, so $\bounce(\iota(P))=\bounce(P)$.

    Thus, for every pair $(a,b)$, the number of polyominoes with $(\area,\bounce)=(a,b)$ cannot decrease when $m$ increases by $1$.
    In other words, every coefficient of $N_{m,k}(q,t)$ is nondecreasing in $m$, and the difference $N_{m,k}(q,t)-N_{m-1,k}(q,t)$ has nonnegative coefficients.
\end{proof}

\subsection{From Polyominoes to Noncrossing Partitions}

Noncrossing partitions $\NC(n, k)$ and parallelogram polyominoes $\PP(n-k+1,k)$ are both counted by the Narayana numbers.
A bridge between these two sets of combinatorial objects is given by the following bijection (which, to the best of our knowledge, is not present in the literature).

\begin{definition}
    Let $\eta \colon \PP(n-k+1,k) \rightarrow \NC(n,k)$ be defined as follows.
    Label from $1$ to $n$ the cells adjacent to the floor of a polyomino, going bottom to top, left to right (call these the \emph{indexing labels}).
    Then label the cells adjacent to the ceiling, again bottom to top, left to right, always writing the largest unused indexing label which is present in the same row or below.
    The entries in each row are the blocks of the target noncrossing partition (\Cref{fig:polyominoes-to-noncrossing}).

\end{definition}

 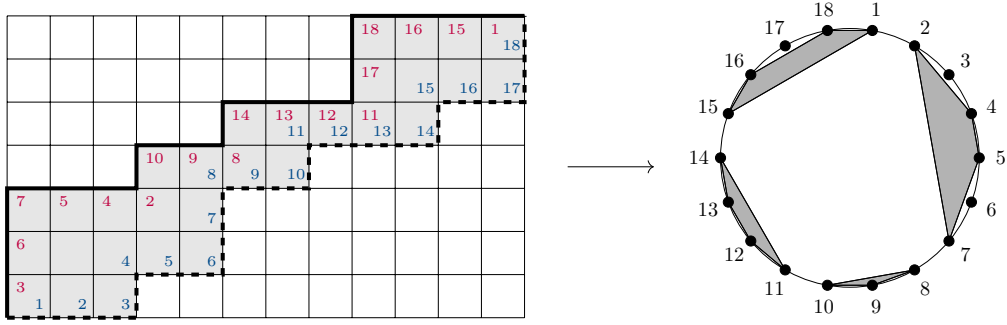
\begin{figure}[!ht]
 	\centering
    \resizebox{0.96\linewidth}{!}{%
 	\begin{tikzpicture}[scale=0.6]
     	\draw[black, thin] (0,0) grid (12,7);

     	\filldraw[black, opacity=0.1] (0,0) -- (3,0) -- (3,1) -- (5,1) -- (5,3) -- (7,3) -- (7,4) -- (10,4) -- (10,5) -- (12,5) -- (12,7) -- (8,7) -- (8,5) -- (5,5) -- (5,4) -- (3,4) -- (3,3) -- (0,3) -- cycle;

     	\draw[dashed, line width=1.6pt] (0,0) -- (3,0) -- (3,1) -- (5,1) -- (5,3) -- (7,3) -- (7,4) -- (10,4) -- (10,5) -- (12,5) -- (12,7);

     	\draw[line width=1.6pt] (0,0) -- (0,3) -- (3,3) -- (3,4) -- (5,4) -- (5,5) -- (8,5) -- (8,7) -- (12,7);

        \begin{scope}[idxlab]
            \node[above] at (0.75, 0) {$1$};
            \node[above] at (1.75, 0) {$2$};
            \node[above] at (2.75, 0) {$3$};
            \node[above] at (2.75, 1) {$4$};
            \node[above] at (3.75, 1) {$5$};
            \node[above] at (4.75, 1) {$6$};
            \node[above] at (4.75, 2) {$7$};
            \node[above] at (4.75, 3) {$8$};
            \node[above] at (5.75, 3) {$9$};
            \node[above] at (6.7, 3) {$10$};
            \node[above] at (6.7, 4) {$11$};
            \node[above] at (7.7, 4) {$12$};
            \node[above] at (8.7, 4) {$13$};
            \node[above] at (9.7, 4) {$14$};
            \node[above] at (9.7, 5) {$15$};
            \node[above] at (10.7, 5) {$16$};
            \node[above] at (11.7, 5) {$17$};
            \node[above] at (11.7, 6) {$18$};
        \end{scope}

        \begin{scope}[toplab]
            \node[right] at (0,0.7) {$3$};
            \node[right] at (0,1.7) {$6$};
            \node[right] at (0,2.7) {$7$};
            \node[right] at (1,2.7) {$5$};
            \node[right] at (2,2.7) {$4$};
            \node[right] at (3,2.7) {$2$};
            \node[right] at (3,3.7) {$10$};
            \node[right] at (4,3.7) {$9$};
            \node[right] at (5,3.7) {$8$};
            \node[right] at (5,4.7) {$14$};
            \node[right] at (6,4.7) {$13$};
            \node[right] at (7,4.7) {$12$};
            \node[right] at (8,4.7) {$11$};
            \node[right] at (8,5.7) {$17$};
            \node[right] at (8,6.7) {$18$};
            \node[right] at (9,6.7) {$16$};
            \node[right] at (10,6.7) {$15$};
            \node[right] at (11,6.7) {$1$};
        \end{scope}

        \draw[->] (13, 3.5) -- (15, 3.5);
 	\end{tikzpicture}
    \hspace{0.21cm}
    \begin{tikzpicture}
        \disk[0,0][1.8]{18}{3}{6}{2,4,5,7}{8,9,10}{11,12,13,14}{17}{1,15,16,18} \stop
    \end{tikzpicture}
    }
 	\caption{An example of the bijection $\eta$.}
	\label{fig:polyominoes-to-noncrossing}
 \end{figure}

We omit the proof that $\eta$ is a bijection.
This bijection can be used to translate statistics from polyominoes to noncrossing partitions and vice versa; this is how we defined area and bounce on noncrossing partitions (appearing in \Cref{fig:stat-atlas}).
For the area, we have the following explicit formula on $\NC(n, k)$ for all $n \geq k \geq 1$:
\[ \area(\pi) = \sum_{\beta, \beta' \in \pi} \# \{i \in \beta' \mid i < \max \beta < \max \beta' \}. \]
This is closely related to the known \textit{inv} statistic on ordered set partitions \cite{RemmelWilson2015}.

The skip statistic from \eqref{eq:skip} can also be read directly from the drawing of a noncrossing partition.
If $\beta=\{b_1<\dots<b_r\}$ is a block, draw the edges joining consecutive elements in increasing order, namely
$b_1b_2,b_2b_3,\dots,b_{r-1}b_r$, and ignore only the closing edge $b_rb_1$.
Each counted edge $b_ib_{i+1}$ contributes the number of labels strictly between its endpoints, namely $b_{i+1}-b_i-1$.
Summing these contributions over all blocks recovers the skip statistic, because
\[
    \sum_{i=1}^{r-1}(b_{i+1}-b_i-1)=b_r-b_1-r+1=\max\beta-\min\beta-|\beta|+1.
\]
An example is shown in \Cref{fig:skip-on-nc}.

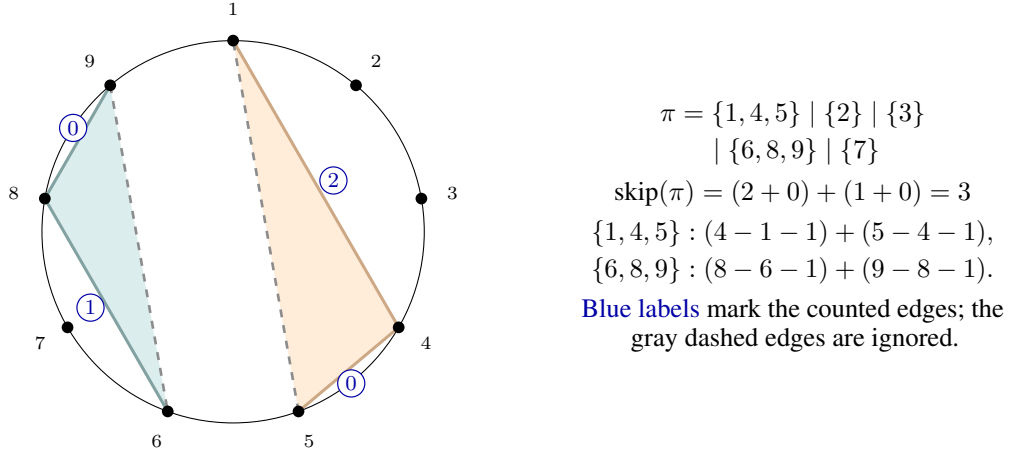
\begin{figure}[ht]
  \centering
  \begin{minipage}[c]{0.54\textwidth}
    \centering
    \begin{tikzpicture}[scale=1.16, every node/.style={font=\scriptsize}]
      \def\R{2.18}
      \colorlet{skipblue}{blue!65!black}
      \colorlet{skipgray}{black!45}
      \colorlet{skiporange}{orange!35}
      \colorlet{skipteal}{teal!32}
      \tikzset{
        skipcount/.style={
          circle,
          draw=skipblue,
          fill=white,
          line width=0.45pt,
          inner sep=1.25pt,
          font=\bfseries\footnotesize,
          text=skipblue
        }
      }

      \foreach \i [evaluate=\i as \ang using 90-40*(\i-1)] in {1,...,9} {
        \coordinate (p\i) at (\ang:\R);
      }

      \fill[skiporange, opacity=0.42] (p1) -- (p4) -- (p5) -- cycle;
      \fill[skipteal, opacity=0.42] (p6) -- (p8) -- (p9) -- cycle;

      \draw[thin] (0,0) circle (\R);

      \draw[skiporange!80!black, line width=1.15pt] (p1) -- (p4) -- (p5);
      \draw[skipgray, dashed, line width=1pt] (p5) -- (p1);

      \draw[skipteal!75!black, line width=1.15pt] (p6) -- (p8) -- (p9);
      \draw[skipgray, dashed, line width=1pt] (p9) -- (p6);

      \foreach \i [evaluate=\i as \ang using 90-40*(\i-1)] in {1,...,9} {
        \fill (p\i) circle (1.9pt);
        \node at (\ang:{\R+0.36}) {$\i$};
      }

      \node[skipcount] at ($(p1)!0.5!(p4)+(0.20,0.04)$) {$2$};
      \node[skipcount] at ($(p4)!0.5!(p5)+(0.03,-0.16)$) {$0$};
      \node[skipcount] at ($(p6)!0.5!(p8)+(-0.18,-0.03)$) {$1$};
      \node[skipcount] at ($(p8)!0.5!(p9)+(-0.05,0.16)$) {$0$};
    \end{tikzpicture}
  \end{minipage}
  \hfill
  \begin{minipage}[c]{0.40\textwidth}
    \[
      \begin{gathered}
        \pi=\{1,4,5\}\mid\{2\}\mid\{3\}\\
        {}\mid\{6,8,9\}\mid\{7\}
      \end{gathered}
    \]
    \[
      \skipstat(\pi)=(2+0)+(1+0)=3
    \]
    \[
      \begin{aligned}
        \{1,4,5\} &: (4-1-1)+(5-4-1), \\
        \{6,8,9\} &: (8-6-1)+(9-8-1).
      \end{aligned}
    \]
    \centering
    \textcolor{blue!65!black}{Blue labels} mark the counted edges; the gray dashed edges are ignored.
  \end{minipage}
  \caption{Computing the skip statistic from a noncrossing partition. The contribution of each counted edge is the number of labels strictly between its endpoints.}
  \label{fig:skip-on-nc}
\end{figure}

We derived an explicit formula for the translation of the bounce only for $k=3$.
The bounce--area connections in \Cref{fig:stat-atlas} are obtained by running our ML methods on the \textit{flipped bounce}, i.e., the bounce of the noncrossing partition reflected in the segment $[1, n]$.
This flip is convenient in the refined setting (introduced in \Cref{sec:skip-testbed}): unlike the usual bounce, the flipped bounce is compatible with the filtration by $m(\pi)$ in the sense that it is unchanged when appending singleton blocks at the end.
Working with the flipped bounce is essentially equivalent, since there is a simple bijection on $\NC(n, 3)$ that fixes the area and exchanges bounce with flipped bounce; we formalize this bijection in Lean (see supplementary code).

\subsection{\texorpdfstring{Exchanging Bijection for $N_{n, 3}(q, t)$}{Exchanging Bijection for N(n,3)(q,t)}}

The bijection on $\NC(n, 3)$ introduced in \Cref{sec:bijection} exchanges the skip and leap statistics, as exemplified in \Cref{fig:bijection-n4-k3} for $n=4$.
This constitutes the first combinatorial proof of the symmetry $N_{n, 3}(q, t) = N_{n, 3}(t, q)$ for all $n \geq 3$, as verified through formal Lean proofs (\Cref{sec:lean-appendix}).
Our bijection preserves the block sizes and permutes the gaps between numbers in non-singleton blocks.

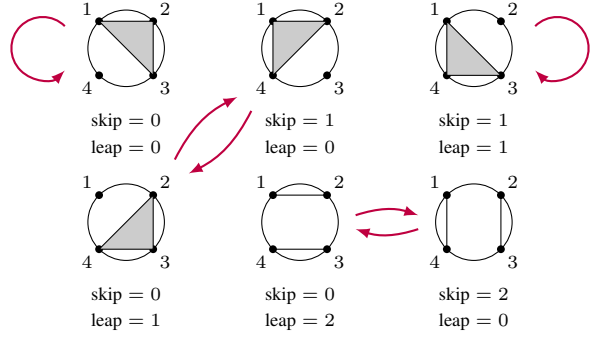
\begin{figure}[ht]
    \centering
    \begin{tikzpicture}[scale=0.92, every node/.style={font=\scriptsize}]
    \def\r{0.55}

    \newcommand{\basefour}{
      \draw (0,0) circle (\r);

      \coordinate (v1) at ( 135:\r);
      \coordinate (v2) at (  45:\r);
      \coordinate (v3) at ( -45:\r);
      \coordinate (v4) at (-135:\r);

      \fill (v1) circle (1.6pt);
      \fill (v2) circle (1.6pt);
      \fill (v3) circle (1.6pt);
      \fill (v4) circle (1.6pt);

      \node[anchor=south east] at ($(0,0)!0.9!(v1)$) {$1$};
      \node[anchor=south west] at ($(0,0)!0.9!(v2)$) {$2$};
      \node[anchor=north west] at ($(0,0)!0.9!(v3)$) {$3$};
      \node[anchor=north east] at ($(0,0)!0.9!(v4)$) {$4$};
    }

    \begin{scope}[xshift=0.0cm, yshift=1.25cm]
      \basefour
      \draw[fill=black!20] (v1)--(v2)--(v3)--cycle;
      \node[below=0pt] at (0,-0.80) {$\skipstat = 0$};
      \node[below=10pt] at (0,-0.80) {$\leap = 0$};
      \coordinate (A) at (-1.2,0);
    \end{scope}

    \begin{scope}[xshift=2.5cm, yshift=1.25cm]
      \basefour
      \draw[fill=black!20] (v1)--(v2)--(v4)--cycle;
      \node[below=0pt] at (0,-0.80) {$\skipstat = 1$};
      \node[below=10pt] at (0,-0.80) {$\leap = 0$};
      \coordinate (B) at (-0.8,-0.8);
    \end{scope}

    \begin{scope}[xshift=5cm, yshift=1.25cm]
      \basefour
      \draw[fill=black!20] (v1)--(v3)--(v4)--cycle;
      \node[below=0pt] at (0,-0.80) {$\skipstat = 1$};
      \node[below=10pt] at (0,-0.80) {$\leap = 1$};
      \coordinate (C) at (1.2,0);
    \end{scope}

    \begin{scope}[xshift=0.0cm, yshift=-1.25cm]
      \basefour
      \draw[fill=black!20] (v2)--(v3)--(v4)--cycle;
      \node[below=0pt] at (0,-0.80) {$\skipstat = 0$};
      \node[below=10pt] at (0,-0.80) {$\leap = 1$};
      \coordinate (D) at (0.8,0.8);
    \end{scope}

    \begin{scope}[xshift=2.5cm, yshift=-1.25cm]
      \basefour
      \draw (v1)--(v2);
      \draw (v3)--(v4);
      \node[below=0pt] at (0,-0.80) {$\skipstat = 0$};
      \node[below=10pt] at (0,-0.80) {$\leap = 2$};
      \coordinate (E) at (0.8,0);
    \end{scope}

    \begin{scope}[xshift=5cm, yshift=-1.25cm]
      \basefour
      \draw (v1)--(v4);
      \draw (v2)--(v3);
      \node[below=0pt] at (0,-0.80) {$\skipstat = 2$};
      \node[below=10pt] at (0,-0.80) {$\leap = 0$};
      \coordinate (F) at (-0.8,0);
    \end{scope}

    \begin{scope}[thick,purple]
        \draw[-latex] (A) ++(45:4.5mm) arc (45:315:4.5mm);
        \draw[latex-] (C) ++(225:4.5mm) arc (225:225+270:4.5mm);

        \draw[-latex] ($(B) + (0.1, -0.1)$) to[bend left=18] ($(D) + (0.1, -0.1)$);
        \draw[-latex] ($(D) + (-0.1, 0.1)$) to[bend left=18] ($(B) + (-0.1, 0.1)$);

        \draw[-latex] ($(E) + (0, 0.1)$) to[bend left=18] ($(F) + (0, 0.1)$);
        \draw[-latex] ($(F) + (0, -0.1)$) to[bend left=18] ($(E) + (0, -0.1)$);
\end{scope}

    \end{tikzpicture}
    \caption{Exchanging bijection for $n=4$ and $k=3$.}
    \label{fig:bijection-n4-k3}
\end{figure}

\subsection{The Min Gap Arc Statistic}
\label{sec:mingarc-appendix}

The \emph{min gap arc} statistic ($\mingarc$) is a human-defined extension of the flipped skew statistic discovered by our ML methods on $\NC(n,3)$. It is defined on $\NC(n,k)$ for arbitrary $k$ through the arc-selection procedure below and agrees with flipped skew when $k=3$.
Together with skip, it gives a combinatorial interpretation of $N_{n,k}(q,t)$.

\textbf{Candidate arcs.}
Let $\pi\in\NC(n,k)$ with $h=n-k+1$ blocks. Order the block minima and maxima as
\[
1=m_1<\cdots<m_h, \qquad c_1<\cdots<c_h=n,
\]
and set $m_{h+1}:=n+1$. For each block $r$ let $\nu_r=m_{r+1}-m_r-1$.
List the elements of $[n]$ that are \emph{not} block maxima in increasing order as $x_1<\cdots<x_{k-1}$ (the \emph{left endpoints}), and form the \emph{right endpoint} list by repeating $c_r$ exactly $\nu_r$ times for each $r$ in sequence, yielding a weakly increasing list $y_1\le\cdots\le y_{k-1}$.
The \emph{candidate arcs} are the intervals $[x_i, y_i]$.

\textbf{Greedy choice.}
Scan the candidate arcs left to right.
Keep the first arc; thereafter always keep the earliest arc whose left endpoint is greater than the right endpoint of the last kept arc.
The \emph{selected arcs} are the ones kept by this greedy procedure, and $\mathcal{I}\subseteq\{1,\dots,k-1\}$ is the set of their indices.

\textbf{Statistic.}
We define $\mingarc$ by summing the gaps from the right endpoints of the selected arcs to $n$:
\[
  \mingarc(\pi) = \sum_{i \in \mathcal{I}} (n-y_i).
\]

The (skip, mingarc) pairing yields the $q,t$-Narayana polynomial:
\[
  N_{n,k}(q,t)=\sum_{\pi\in\NC(n,k)} q^{\skipstat(\pi)}\, t^{\mingarc(\pi)}.
\]

\begin{figure}[ht]
  \centering
  \begin{subfigure}[t]{0.48\linewidth}
    \centering
    \begin{tikzpicture}[scale=0.92, every node/.style={font=\scriptsize}]
      \def\r{0.62}
      \tikzset{
        ncblock/.style={black!70, line width=0.9pt},
        candarc/.style={black!35, line width=0.85pt, dashed, line cap=round},
        chosenarc/.style={BrickRed, line width=1.25pt, line cap=round}
      }
      \newcommand{\basesix}{
        \coordinate (c0) at (0,0);
        \draw (0,0) circle (\r);
        \coordinate (v1) at (120:\r);
        \coordinate (v2) at (60:\r);
        \coordinate (v3) at (0:\r);
        \coordinate (v4) at (-60:\r);
        \coordinate (v5) at (-120:\r);
        \coordinate (v6) at (180:\r);
        \fill (v1) circle (1.5pt);
        \fill (v2) circle (1.5pt);
        \fill (v3) circle (1.5pt);
        \fill (v4) circle (1.5pt);
        \fill (v5) circle (1.5pt);
        \fill (v6) circle (1.5pt);
        \node[anchor=south east] at ($(0,0)!1.18!(v1)$) {$1$};
        \node[anchor=south west] at ($(0,0)!1.18!(v2)$) {$2$};
        \node[anchor=west] at ($(0,0)!1.18!(v3)$) {$3$};
        \node[anchor=north west] at ($(0,0)!1.18!(v4)$) {$4$};
        \node[anchor=north east] at ($(0,0)!1.18!(v5)$) {$5$};
        \node[anchor=east] at ($(0,0)!1.18!(v6)$) {$6$};
      }
      \newcommand{\circlearc}[6]{%
        \pgfmathsetmacro{\arcR}{(#4)*(#5)*\r}
        \pgfmathsetmacro{\gammaang}{asin(1/(#5))}
        \pgfmathsetmacro{\startang}{(#2)+180+(#3)*\gammaang}
        \pgfmathsetmacro{\endang}{(#2)+180-(#3)*\gammaang}
        \draw[#6] (#1) arc[start angle=\startang, end angle=\endang, radius=\arcR];
      }

      \begin{scope}[xshift=0cm]
        \node at (0,1.05) {};
        \basesix
        \draw[ncblock] (v1) -- (v6);
        \draw[ncblock] (v2) -- (v3);
        \draw[ncblock] (v4) -- (v5);
      \end{scope}

      \begin{scope}[xshift=2.35cm]
        \node at (0,1.05) {};
        \basesix
        \draw[chosenarc] (v1) to[bend left=42] (v5);
        \draw[candarc] (v2) to[bend left=28] (v6);
        \draw[candarc] (v4) to[bend right=28] (v6);
      \end{scope}
    \end{tikzpicture}
    \caption{$\pi=16\mid23\mid45$. The selected arc is $1$--$5$; the candidate arcs $2$--$6$ and $4$--$6$ are not selected.}
  \end{subfigure}
  \hfill
  \begin{subfigure}[t]{0.48\linewidth}
    \centering
    \begin{tikzpicture}[scale=0.92, every node/.style={font=\scriptsize}]
      \def\r{0.62}
      \tikzset{
        ncblock/.style={black!70, line width=0.9pt},
        candarc/.style={black!35, line width=0.85pt, dashed, line cap=round},
        chosenarc/.style={BrickRed, line width=1.25pt, line cap=round}
      }
      \newcommand{\basesix}{
        \coordinate (c0) at (0,0);
        \draw (0,0) circle (\r);
        \coordinate (v1) at (120:\r);
        \coordinate (v2) at (60:\r);
        \coordinate (v3) at (0:\r);
        \coordinate (v4) at (-60:\r);
        \coordinate (v5) at (-120:\r);
        \coordinate (v6) at (180:\r);
        \fill (v1) circle (1.5pt);
        \fill (v2) circle (1.5pt);
        \fill (v3) circle (1.5pt);
        \fill (v4) circle (1.5pt);
        \fill (v5) circle (1.5pt);
        \fill (v6) circle (1.5pt);
        \node[anchor=south east] at ($(0,0)!1.18!(v1)$) {$1$};
        \node[anchor=south west] at ($(0,0)!1.18!(v2)$) {$2$};
        \node[anchor=west] at ($(0,0)!1.18!(v3)$) {$3$};
        \node[anchor=north west] at ($(0,0)!1.18!(v4)$) {$4$};
        \node[anchor=north east] at ($(0,0)!1.18!(v5)$) {$5$};
        \node[anchor=east] at ($(0,0)!1.18!(v6)$) {$6$};
      }
      \newcommand{\circlearc}[6]{%
        \pgfmathsetmacro{\arcR}{(#4)*(#5)*\r}
        \pgfmathsetmacro{\gammaang}{asin(1/(#5))}
        \pgfmathsetmacro{\startang}{(#2)+180+(#3)*\gammaang}
        \pgfmathsetmacro{\endang}{(#2)+180-(#3)*\gammaang}
        \draw[#6] (#1) arc[start angle=\startang, end angle=\endang, radius=\arcR];
      }

      \begin{scope}[xshift=0cm]
        \node at (0,1.05) {};
        \basesix
        \draw[ncblock] (v1) -- (v2);
        \draw[ncblock] (v3) -- (v4);
        \draw[ncblock] (v5) -- (v6);
      \end{scope}

      \begin{scope}[xshift=2.35cm]
        \node at (0,1.05) {};
        \basesix
        \circlearc{v1}{90}{-1}{0.5}{1.25}{candarc}
        \circlearc{v3}{-30}{-1}{0.5}{1.25}{candarc}
        \circlearc{v5}{-150}{-1}{0.5}{1.25}{candarc}
        \circlearc{v1}{90}{-1}{0.5}{1.25}{chosenarc}
        \circlearc{v3}{-30}{-1}{0.5}{1.25}{chosenarc}
        \circlearc{v5}{-150}{-1}{0.5}{1.25}{chosenarc}
      \end{scope}
    \end{tikzpicture}
    \caption{$\pi=12\mid34\mid56$. The candidate arcs are disjoint, so all of them are selected.}
  \end{subfigure}
  \caption{Min gap arc construction on two noncrossing partitions. Solid red arcs are selected; dashed gray arcs are candidate arcs that are not selected.}
  \label{fig:mingarc-arcs}
\end{figure}
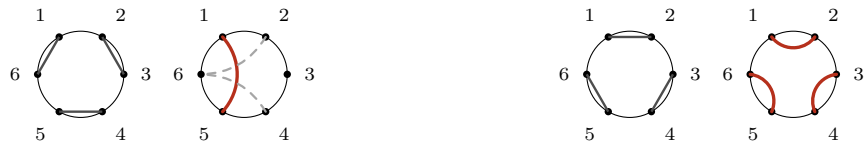

\textbf{The case $k=3$.}
There are exactly two candidate arcs.
They overlap when the partition contains a triple block or a nested pair; in the non-nested case they are disjoint and both are kept.
The statistic therefore evaluates to
\[
  \mingarc(abc) = n-b-1,
  \qquad
  \mingarc(ab \mid cd) = 2n-b-d \quad (a<b<c<d),
\]
and in the nested case
\[
  \mingarc(ab \mid cd) = n-d-1 \qquad (a<c<d<b).
\]
These are exactly the values of \eqref{eq:skew} composed with the flip $i\mapsto n+1-i$.
Hence, on $\NC(n,3)$, the statistic $\mingarc$ is precisely the flipped skew.

\section{Formal Lean Proofs}
\label{sec:lean-appendix}

Lean~4 is an interactive theorem prover (proof assistant) and functional programming language for writing machine-checked mathematics.
In Lean, mathematical statements are encoded as types and proofs as terms inhabiting those types; correctness is verified by a small trusted kernel via type checking.
During development, Lean also supports the keyword \texttt{sorry}, which acts as an explicit placeholder for a missing proof term: it allows the file to type-check while recording a proof gap.
We refer to \citep{the_lean4_paper} for a detailed description of the system.
Our formalization builds on Mathlib, the community-maintained library of definitions and theorems for Lean \cite{mathlib_cpp2020}.

To accelerate the formalization effort, we made extensive use of \emph{Aristotle}, Harmonic's Lean formalization system \cite{aristotle}, which closes proof gaps and autoformalizes informal arguments.

This appendix gives a brief guide to the Lean~4 code accompanying this paper.
Our goal in using Lean is \emph{verification}: the definitions and proofs that certify the main combinatorial claims are checked by Lean's kernel, making the verification independently reproducible.

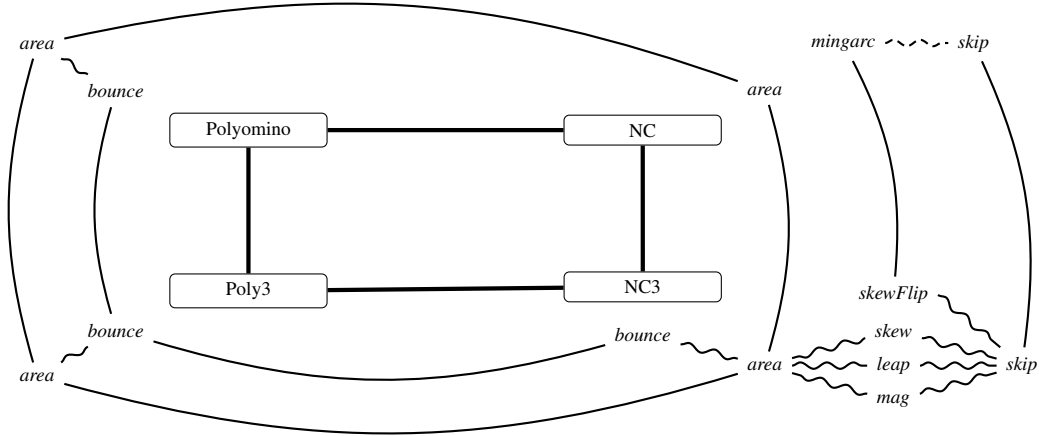
\begin{figure}[h]
  \centering
  \begin{tikzpicture}[scale=1.04, transform shape,
    font=\footnotesize,
    >=Latex,
    node distance=16mm and 32mm,
    box/.style={draw, rounded corners=2pt, inner sep=3pt, minimum width=20mm, align=center, font=\scriptsize},
    stat/.style={font=\scriptsize\itshape},
    equiv/.style={ultra thick},
    incl/.style={ultra thick},
    trans/.style={thick},
    conj/.style={dotted, thick},
    lift/.style={thick, bend left=16},
    lift2/.style={thick, bend right=16},
    squig/.style={decorate, decoration={snake, amplitude=1.2pt, segment length=10pt}, thick},
    special/.style={double, thick}
  ]

  \node[box] (poly) {Polyomino};
  \node[box, below=of poly] (poly3) {Poly3};

  \node[box, right=30mm of poly] (nc) {NC};
  \node[box, below=of nc] (nc3) {NC3};

  \draw[equiv] (poly) -- (nc);
  \draw[incl] (poly) -- (poly3);
  \draw[incl] (nc) -- (nc3);
  \draw[incl] (poly3) -- (nc3);

  \node[stat, left=2mm of poly, yshift=5mm] (polyBounce) {bounce};
  \node[stat, above=2mm of polyBounce, xshift=-10mm] (polyArea) {area};
  \node[stat, left=2mm of poly3, yshift=-5mm] (poly3Bounce) {bounce};
  \node[stat, below=2mm of poly3Bounce, xshift=-10mm] (poly3Area) {area};

  \draw[squig] (polyArea) to (polyBounce);

  \draw[trans,lift2]  (polyArea) to (poly3Area);
  \draw[trans,lift2] (polyBounce) to (poly3Bounce);

  \draw[squig] (poly3Area) to (poly3Bounce);

  \node[stat, right=2mm of nc, yshift=5mm] (ncArea) {area};
  \node[stat, above=2mm of ncArea, xshift=10mm] (ncMingarc) {mingarc};
  \node[stat, right=8mm of ncMingarc] (ncSkip) {skip};

  \draw[squig, densely dashed] (ncMingarc) to (ncSkip);

  \node[stat, below=2mm of nc3] (nc3Bounce) {bounce};

  \draw[trans, lift2] (poly3Bounce) to (nc3Bounce);

  \draw[trans, lift] (polyArea) to (ncArea);

  \node[stat, right=2mm of nc3, yshift=-10mm] (statArea) {area};
  \node[stat, right=26mm of statArea] (statSkip) {skip};
  \node[stat] (statLeap) at ($(statArea)!0.5!(statSkip)$) {leap};
  \node[stat] (statSkew) at ($(statArea)!0.5!(statSkip) + (0, 4.5mm)$) {skew};
  \node[stat] (statSkewFlip) at ($(statArea)!0.5!(statSkip) + (0, 9mm)$) {skewFlip};
  \node[stat] (mag) at ($(statArea)!0.5!(statSkip) + (0, -4.5mm)$) {mag};

  \draw[squig] (statArea) -- (statLeap);
  \draw[squig] (statLeap) -- (statSkip);
  \draw[squig] (statSkew) -- (statArea);
  \draw[squig] (statSkew) -- (statSkip);
  \draw[squig] (statSkewFlip.east) -- (statSkip);

  \draw[trans, lift2] (poly3Area) to (statArea);
  \draw[squig] (nc3Bounce) to (statArea);
  \draw[trans, lift] (ncSkip) to (statSkip);
  \draw[trans, lift] (ncArea) to (statArea);
  \draw[trans, lift] (ncMingarc) to (statSkewFlip);

  \draw[squig] (mag) -- (statArea);
  \draw[squig] (mag) -- (statSkip);

  \end{tikzpicture}
  \caption{Overview of the objects and statistics covered by the Lean code. Thick edges indicate equivalences of types, and are intended with specific $m$, $n$, and $k$ values.
  Curved thin edges indicate transport of statistics across these maps; squiggly edges indicate pairing relationships among statistics.
  All edges shown correspond to statements whose proofs are formalized in the current Lean development.
  }
  \label{fig:appendix-stat-atlas}
\end{figure}

At a high level, the development follows the structure summarized in Figure~\ref{fig:appendix-stat-atlas}.
Two of the boxes (Polyomino and NC) correspond to the main mathematical objects appearing in this paper: parallelogram polyominoes and noncrossing partitions. Both have been formalized in Lean in full generality, for arbitrary parameters $m$, $n$, and $k$.
Since our main results concern the case $k=3$ (equivalently, $m=3$), we also introduce two specialized definitions (Poly3 and NC3) for those instances.
In these specialized representations, objects can be described by three or four natural numbers and are easier to manipulate algebraically.
To ensure correctness, we connect the specialized encodings to the general definitions via explicit bijections.
The Lean code formalizes (i) maps between these objects and representations and (ii) the way statistics are translated across them.
It also formalizes a number of statistic pairings, including the ``atlas'' connections on NC3 (\Cref{fig:stat-atlas}) and the all-$k$ \texttt{skip}--\texttt{mingarc} pairing on \texttt{NC}.

\subsection{What is verified}
We formalize all the mathematical results found in the paper, as well as the necessary mathematical objects needed to state them and the basic combinatorial theory around them.
For a reader-oriented entry point, we recommend opening \texttt{qtLearning.lean}, which is a curated index of the main definitions and theorems (in VS Code, the imported filenames and \texttt{\#check} declarations are clickable once the file compiles).
For instructions on building the project and running the artifact locally, see the file \texttt{README.md}.
\begin{itemize}[leftmargin=*]
  \item \textbf{Representations and maps.} Definitions of the main types (\texttt{Polyomino} and \texttt{NC}) and their specializations (\texttt{Poly3} and \texttt{NC3}), together with the key connecting maps (\texttt{Phi}/\texttt{Psi}, \texttt{NC.equivNC3}, \texttt{Poly3.equivNC3}, \texttt{Poly3.equivPolyomino}).
  The forward/backward maps are implemented throughout, and the stated equivalences are fully proved, including the general Polyomino $\leftrightarrow$ NC equivalence \texttt{Poly\_NC\_equiv} built from \texttt{Phi} and \texttt{Psi}.

  \item \textbf{Statistics.} Definitions of the statistics used in the paper (\texttt{area}, \texttt{bounce}, \texttt{bounceFlip}, \texttt{skip}, \texttt{leap}, \texttt{skew}, \texttt{skewFlip}, \texttt{mag}, \texttt{mingarc}) and lemmas expressing compatibility under the representation maps (e.g., \texttt{NC3.toNC\_skip}, \texttt{NC.toNC3\_skip}, \texttt{NC3.toNC\_area}).
  All compatibility lemmas needed for the $k=3$ results are proved, and the full Polyomino-to-NC transport of \texttt{area} through \texttt{Phi}/\texttt{Psi} is proved as \texttt{Polyomino.Phi\_area} and \texttt{NC.Psi\_area}.

  \item \textbf{Pairings.} We formalize Narayana pairings by expressing that two pairs of statistics have the same joint distribution as $(\area,\bounce)$ on \texttt{Polyomino}.
  The atlas pairings among \texttt{area}, \texttt{skip}, \texttt{leap}, \texttt{skew}, \texttt{mag}, and \texttt{skewFlip} are fully verified on \texttt{NC3}; the all-$k$ \texttt{skip}--\texttt{mingarc} pairing on \texttt{NC} is proved as \texttt{NC.skip\_mingarc\_pair}.
  The pairings on \texttt{Polyomino} and on \texttt{Poly3} are also formalized.
  It is noteworthy that the proofs of the core $k=3$ pairing theorems (e.g., \texttt{NC3.skip\_leap\_pair} and related atlas pairings) were found and formalized end-to-end by Aristotle, without human intervention.

  \item \textbf{Exchanging bijections and $q,t$-symmetry.} We formalize explicit bijections on \texttt{NC3} that exchange pairs of statistics (e.g., \texttt{skip}$\leftrightarrow$\texttt{leap}).
  In particular, the Lean development includes a fully formal, kernel-checked bijective proof of the symmetry $N_{n,3}(q,t)=N_{n,3}(t,q)$ for all $n\ge 3$, which is a special case of \Cref{prob:narayana-bijection}.

  \item \textbf{Narayana polynomials.} A lightweight interface for $q,t$-generating functions of statistics, including \texttt{qtPoly} and \texttt{EquallyDistributed\textsubscript{2}}.
  We also define the $q,t$-Narayana polynomial \texttt{Narayana.qtPoly} as the $(\area,\bounce)$ generating function on \texttt{Polyomino}.
  In the Lean code, this polynomial is parameterized using the polyomino convention $(m,n)$ rather than the noncrossing convention $(n,k)$.
  We intentionally implement this interface as functions $\mathbb{N}\to\mathbb{N}\to\mathbb{N}$ (rather than using Mathlib's polynomial types), since we only need coefficient-level identities coming from finite enumerations.
  The formalized interface includes the $m=3$ symmetry theorem \texttt{Narayana.qtPoly\_3\_symm} and the coefficientwise monotonicity theorem \texttt{Narayana.incrementalNarayanaPositive} used in \Cref{prop:incremental-narayana-positive}.

\end{itemize}

Finally, we include some test files, where we define specific instances of polyominoes and noncrossing partitions and check their statistics.
This is possible because most of our implementation (including the statistics) is computable, so we can run the calculations in concrete cases.
We also include tests comparing the formally computed \texttt{Narayana.qtPoly} against generator-produced ground truth up to $n=14$, along with formal proofs that the two sources agree.

\subsection{Formalization workflow}
While Lean provides the final verification guarantee (kernel-checked definitions and proofs), the development process was highly iterative.
We made extensive use of Aristotle in two complementary modes:
\begin{itemize}[leftmargin=*]
  \item \textbf{Sorry-filling mode.} During development, some files temporarily contained \texttt{sorry} placeholders.
  Aristotle was used to propose proof terms and tactic scripts to solve these goals.
  Candidate proofs were then refined by hand when needed, mostly for performance and readability reasons.
  The sorry-filling mode was mainly used without informal guidance, relying on Aristotle's proof search capabilities.
  In particular, in the \texttt{NC3} specialization, Aristotle was able to autonomously prove the core pairing theorems end-to-end.
  \item \textbf{Autoformalization mode.}
  This mode takes both a \emph{formal} input (the Lean file being developed, including its imports and local context) and an \emph{informal} input (a text or \LaTeX{} document containing an informal description of the mathematics or a proof sketch), and translates informal mathematics to Lean.
  As with sorry-filling, the resulting artifacts were iterated on when needed to align definitions, make implicit conditions explicit, and fit proofs to the existing library interface.
\end{itemize}

In practice, this created a tight loop between (i) writing or refining the informal arguments, (ii) re-running automation to update the Lean files, (iii) consolidating and polishing the resulting proof scripts, and, when some solvable \texttt{sorry}s were identified, (iv) launching a sorry-filling run to eliminate them.
The end product is still fully checked by Lean, so the use of automation only affects developer productivity, and not the trust model.

\subsection{How to build the artifacts}
To check the implementation locally, follow the instructions in \texttt{README.md}.

The Lean project is pinned to Lean toolchain version \texttt{v4.28.0} (see the \texttt{lean-toolchain} file in the Lean project directory).

To build it, first change into the directory containing \texttt{lakefile.toml}: this is the repository root in a full development checkout, and the \texttt{Lean/} directory in the supplementary artifact. Then run:
\begin{verbatim}
lake exe cache get!
lake build
\end{verbatim}

To check whether a specific theorem statement depends on admitted proofs, Lean provides an ``axioms'' report.
In a Lean file, you can run \texttt{\#print axioms <theoremName>} after the theorem (e.g., on a pairing theorem);
if the output mentions \texttt{sorryAx}, then the theorem depends (directly or indirectly) on a \texttt{sorry}.
It is normal for fully proved results in Mathlib-based developments to depend on a small set of standard axioms such as \texttt{propext}, \texttt{Classical.choice}, and \texttt{Quot.sound}; some computational proofs can also report Lean's trusted reduction/compiler axioms.
These do not indicate missing proofs.

\end{document}